\documentclass{article}
\usepackage{ijcai26}

\usepackage{times}
\usepackage{soul}
\usepackage{url}
\usepackage[hidelinks]{hyperref}
\usepackage[utf8]{inputenc}
\usepackage[small]{caption}
\usepackage{graphicx}
\usepackage{amsmath}
\usepackage{amsthm}
\usepackage{booktabs}
\usepackage{algorithm}
\usepackage[switch]{lineno}

\usepackage{mathtools}
\usepackage{amsthm}
\usepackage{amssymb}
\usepackage{xspace}
\usepackage{siunitx}
\usepackage[noend]{algpseudocode}
\usepackage{tabularx}
\usepackage{makecell}
\usepackage{subcaption}

\newtheorem{theorem}{Theorem}

\newtheorem{proposition}[theorem]{Proposition}

\newtheorem{definition}{Definition}

\newcommand{\algname}{\textsc{Woodelf}\xspace} 
\newcommand{\algnamenew}{\textsc{Woodelf++}\xspace} 
\xspaceaddexceptions{'’}
\newcommand{\sckl}{scikit-learn\xspace}

\newcommand{\PDIV}{\mbox{\textit{PDIV}}}
\newcommand{\PDV}{\mbox{\textit{PDV}}}
\newcommand{\CPDV}{\mbox{\textit{CPDV}}}
\newcommand{\aopdiv}{Any-Order-PDIVs\xspace}

\usepackage{tcolorbox}

\title{\algnamenew: A Fast and Unified Partial Dependence Plot Algorithm for Decision Tree Ensembles}

\author{
Ron Wettenstein$^1$
\and
Alexander Nadel$^2$\and
Udi Boker$^1$\\
\affiliations
$^1$Reichman University, Herzliya, Israel\\
$^2$Faculty of Data and Decision Sciences, Technion, Haifa, Israel\\
\emails
ron.wettenstein@post.runi.ac.il,
alexandernad@technion.ac.il
}

\begin{document}

\maketitle

\begin{abstract}
Partial Dependence Plots (PDPs) visualize how changes in a single feature affect the average model prediction. They are widely used in practice to interpret decision tree ensembles and other machine learning models. Joint-PDPs extend this idea to pairs of features, revealing their combined effect.

Partial Dependence Interaction Values (PDIVs) measure feature interactions. The \aopdiv task computes these interactions for every feature subset across all rows of the dataset.

We introduce \algnamenew, a unified and efficient approach for computing all these useful explainability tools on decision tree ensembles, building on 
\algname, an algorithm for efficient SHAP computation. By deriving suitable metrics over pseudo-Boolean functions, \algnamenew can compute PDPs (exact and approximate), Joint-PDPs, and \aopdiv in a unified framework.

Our method delivers substantial complexity improvements over the state of the art, including an exponential gain for \aopdiv. Additionally, we introduce and efficiently compute \emph{Full PDPs}, which leverage the model’s split thresholds to faithfully capture its behavior across all possible feature values.

\algnamenew is implemented in pure Python\footnote{\algnamenew was migrated to the \textsc{Woodelf} GitHub package: \url{https://github.com/ron-wettenstein/woodelf} \\ This paper is an extended version of the work to appear at IJCAI-26.} and supports GPU acceleration. 
On a dataset with \num{400000} rows, \algnamenew computes PDP and Joint-PDP up to $6\times$ faster than the state of the art and up to five orders of magnitude faster than \sckl. For \aopdiv, the gap is even larger: \algnamenew computes all interaction values in 5 minutes, while the state of the art is estimated to require over \num{1000000} years.

\end{abstract}
\section{Introduction}

As decision tree ensembles are widely used in real-world applications~\cite{decision_trees_are_usefull}, understanding their predictions is crucial for trust and decision-making~\cite{decision_trees_trustworthiness,should_i_trust_you}.
This paper accelerates the computation of three widely used attribution metrics for decision tree ensembles: Partial Dependence Plot, Joint Partial Dependence Plot, and Partial Dependence Interaction Value. 

\begin{figure}[t]
  \centering
  \begin{subfigure}{0.49\columnwidth}
    \centering
    \includegraphics[width=\linewidth]{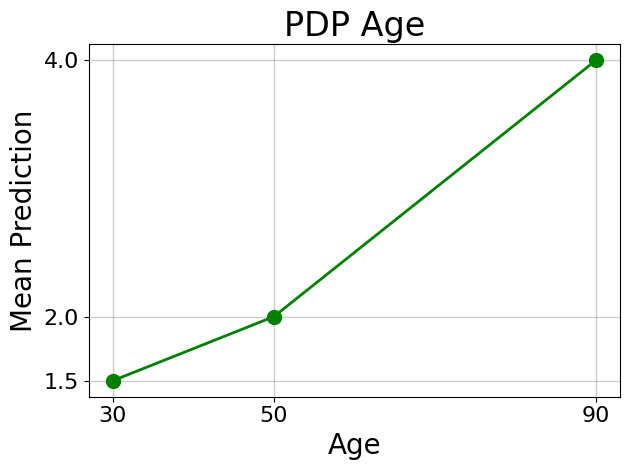}
    \caption{\label{fig:pdp_age}PDP on \textit{age}, include three possible values: 30, 50 and 90.}
  \end{subfigure}
  \hfill
  \begin{subfigure}{0.49\columnwidth}
    \centering
    \includegraphics[width=\linewidth]{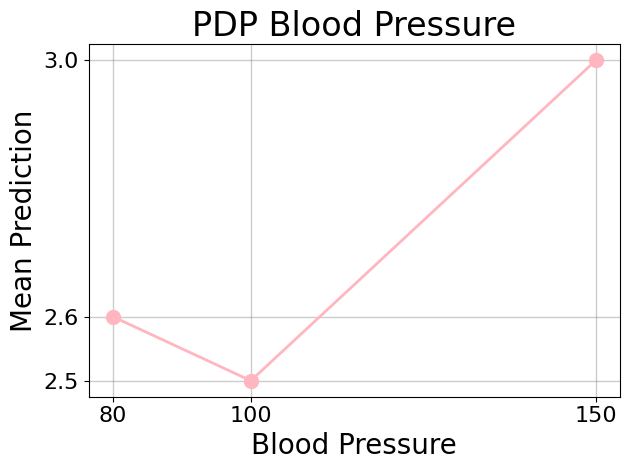}
    \caption{\label{fig:pdp_blood_pressure}PDP on \textit{blood pressure}, with possible values 80, 100 and 150.}
  \end{subfigure}
  \caption{\label{fig:pdp_example} PDP examples.}
\end{figure}

\begin{figure}[t]
    \centering
    \includegraphics[width=0.8\linewidth]{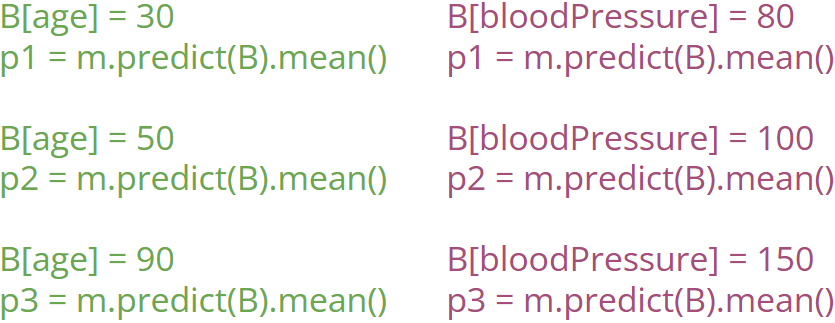}
    \caption{Naive computation of the PDVs required to plot the PDPs in Fig.~\ref{fig:pdp_example}, carried out according to Def.~\ref{def:pdv}.}
    \label{fig:brute_example}
\end{figure}

\subsection{Partial Dependence Plot (PDP)}
\label{sec:intro_PDP}

A PDP~\cite{GreedyFunctionApproximation} illustrates how a model’s average prediction varies with changes in a specific feature. For example, Fig.~\ref{fig:pdp_age} shows that the average prediction increases with \textit{age}. PDPs are widely used in many scientific domains such as healthcare~\cite{GA2M_healthcare}, ecology~\cite{birds_PDP,fish_PDP}, and economics~\cite{PDP_in_housing}.

A PDP is a line plot of several PDVs (Def.~\ref{def:pdv}). 

\begin{definition}[Partial Dependence Value (PDV)] 
\label{def:pdv}
A PDV measures the average prediction of a model $M$ when a feature $f_j$ is fixed to a given value $v$. 
    For a dataset $B$:
\begin{equation}
\label{PDV_formula}
    \PDV_{v, f_j} = 
    \frac{1}{|B|}  \sum_{b \in B} M \left(  \begin{cases} 
    v & \text{if } f_i=f_j  \\
    b[f_i] & \text{otherwise }
    \end{cases} \right)
\end{equation}
\end{definition}

To construct a PDP for feature $f_j$, we select $k$ possible values of $f_j$, compute their PDVs, and plot them. 
The selected values are typically sampled from $B[f_j]$ or evenly spaced across its range.

For each feature $f_a$ and samples $\{a_1, \dots a_k\}$:
\[
\text{PDP requires } \forall_{1\le i \le k}: PDV_{\{a_i\}, \{f_a\}}
\]

The naive way to construct these plots is to compute each PDV separately using Formula~\ref{PDV_formula} (see Fig.~\ref{fig:brute_example}). The overall complexity of computing PDPs for all features using this method is $O(kfnTD)$ (see Fig.~\ref{fig:notations} for notation). 
This naive approach remains the standard in widely used frameworks such as \sckl~\cite{scikit-learn} and \emph{PiML}~\cite{sudjianto2023piml}. 
\begin{figure}[t]
    \centering
    \begin{tcolorbox}
        \label{box:notation}
        $\boldsymbol{k}$: \#selected values to plot;
        $\boldsymbol{f}$: \#features; 
        $\boldsymbol{F}$: Max number of features in a single tree; 
        $\boldsymbol{n}$: \#samples (rows); 
        $\boldsymbol{T}$: \#trees; 
        $\boldsymbol{D}$: tree depth; 
        $\boldsymbol{L}$: leaves per tree. 
    \end{tcolorbox}
    \caption{\label{fig:notations} Notations.}
\end{figure}

~\cite{GreedyFunctionApproximation} also introduced a faster approach that approximates the PDVs, which we discuss in detail in Sec.~\ref{sec:faster_pdp}.

Several extensions of PDPs have been proposed. Individual Conditional Expectation (ICE)~\cite{ICE} plots can be viewed as PDPs computed for a single data row, revealing variation in feature effects across samples. Accumulated Local Effects (ALE)~\cite{ALE} modifies the $PDV$ computation by focusing on rows with feature values similar to $v$. A comparative analysis of PDPs, ALE, and Shapley values is presented in~\cite{PDP_ALE_SHAP_compare}. Additionally, iPDP~\cite{iPDP} handles dynamic models that evolve over time, and StratPD~\cite{StratPD} takes a different approach: instead of interpreting the model, it generates plots that directly explain the training data.

\subsection{Joint Partial Dependence Plot (Joint-PDP)}
\label{sec:intro_joint_PDP}

The concept of Partial Dependence Values extends naturally to feature sets: instead of fixing a single feature, we fix a set of features to chosen values - see Def.~\ref{def:pdv_for_sets}.

\begin{definition}[Partial Dependence Value (PDV) on sets] 
\label{def:pdv_for_sets}
PDV on a subset of features $F=\{f_1, \dots, f_h\}$, set of values $V=\{v_1, \dots, v_h\}$ (write $V[f_i] = v_i$ for the value of $f_i$) and a dataset $B$ is defined as follows:
\begin{equation}
\label{PDV_set_formula}
    \PDV_{V, F} = 
    \frac{1}{|B|}  \sum_{b \in B} M \left(  \begin{cases} 
    V[f_i] & \text{if } f_i \in F  \\
    b[f_i] & \text{otherwise }
    \end{cases} \right)
\end{equation}

\end{definition}


To visualize the joint contribution of a feature pair, it is useful to plot their PDVs. For features $f_a$, $f_b$ with sampled values $\{a_1, \dots, a_h\}$ and 
$\{b_1, \dots, b_h\}$, Joint-PDP computes:
\[
PDV_{\{a_i, b_j\}, \{f_a, f_b\}} \quad \text{for all } 1\le i, j \le h
\]

The naive approach, also used by \sckl, computes the $PDV$ for every feature pair and value by fixing the pair to specific values and evaluating the model prediction (similar to Fig.~\ref{fig:brute_example}, but fixing two features instead of one). This results in a time complexity of $O(k^2 f^2 nTD)$.

\subsection{Partial Dependence Interaction Value (PDIV)}
\label{sec:intro_PDIV}

It is often useful to decompose the joint effect of features into their individual effects and the effect of their interactions. PDIVs (Def.~\ref{def:pdiv}) do exactly this. 

\begin{definition}[Partial Dependence Interaction Value (PDIV)] 
\label{def:pdiv}
A PDIV quantifies the interaction effect among the features $F = \{f_1, \dots, f_h\}$ 
when their values are fixed to $V = \{v_1, \dots, v_h\}$ (with $V[f_i] = v_i$). 
PDIVs on dataset $B$ are defined as follows:
\begin{equation}
\label{PDIV_formula}
    \PDIV_{V, F} = \sum_{S \subseteq F} (-1)^{|F|-|S|} \cdot \PDV_{(V[f_i])_{\forall f_i \in S},S}
\end{equation}
\end{definition}
\begin{equation*}
\begin{split}
    \text{For example: } \PDIV_{\{4, 1, 6\},\{f_2,f_5,f_8\}} = \PDV_{\{4, 1,6\},\{f_2,f_5,f_8\}} \\
    - \PDV_{\{4,1\},\{f_2,f_5\}} - \PDV_{\{1,6\},\{f_5,f_8\}} - \PDV_{\{4,6\},\{f_2,f_8\}}  \\+ \PDV_{\{4\},\{f_2\}} + \PDV_{\{1\},\{f_5\}} + \PDV_{\{6\},\{f_8\}} - \PDV_{\{\},\{\}}
\end{split}
\end{equation*}

Summing all interaction values over subsets of $F$ reconstructs the PDV of $F$. Moreover, \cite{shap_from_pdivs} showed that the Shapley value of a feature $f_i$ is a weighted sum of the PDIVs of all subsets $S$ such that $f_i \in S$.

The \emph{\aopdiv} task computes PDIVs for each row $c$ in the explained data $C$, across all feature subsets. 
Let $c|_S = \{c[f_i] : f_i \in S\}$, then the \emph{\aopdiv} task computes:
\[
PDIV_{c|_S, S} \quad \text{for all } c \in C, \; S \subseteq F
\]

An acceptable result for the \emph{\aopdiv} task is to return all non-zero PDIVs. A naive method runs predictions on the background data $B$ (see Def~\ref{def:pdiv}, for simplicity, we assume $|B| = |C| = n$), 
for every feature subset $S \subseteq F$ ($2^f$ subsets in total), 
and for every explained sample $c \in C$. 
This results in a total complexity of $O(2^f n^2 T D)$.

\begin{table*}[t]
\centering
\renewcommand{\arraystretch}{1.6}
\begin{tabularx}{\textwidth}{l|c|c|c}
\textbf{Task} & \textbf{scikit-learn} & \textbf{FastPD} & \textbf{\algnamenew}
\\\hline
PDP, Sect.~\ref{sec:faster_pdp} & $O(kfnTD)$ & $O(T2^DF(n+k))$ & $O(nTL + kTLD + TL2^D D)$ \\
\makecell[l]{\aopdiv, \\ Sect.~\ref{sec:pdivs}} & - & $O(2^F \!nT)$ & $O(n2^D T L + 4^D T L)$ \\
Joint-PDP, Sect.~\ref{sec:joint_PDP} & $O(k^2f^2nTD)$ & $O(T2^DF^2(n+k^2log(f)))$ & $O(k^2f^2log(f) \!+ \!nTL\! +\!  k^2log(f)TLD^2\! +\! TL2^D D^2)$ \\
\end{tabularx}
\caption{\label{complexity_table} Time complexity of \textsc{scikit-learn}, \textsc{FastPD} and \algnamenew (see Fig.~\ref{fig:notations} for notations).}
\end{table*}
\subsection{FastPD}

FastPD~\cite{FastPD} is the current state-of-the-art tool for computing \aopdiv. It exploits the structure and additivity of decision-tree ensembles to improve performance, achieving two main complexity improvements.
\begin{enumerate}
    \item By processing each tree independently, FastPD reduces $2^f$ to $2^F$. This is beneficial, as the number of features used in a tree ($F$) is typically much smaller than the total number of features ($f$).
    \item FastPD preprocessing reduces the data-size term from $O(mn)$ to $O(m+n)$, where $m=|B|$ and $n=|C|$.
\end{enumerate}
Overall, the complexity of FastPD for computing \aopdiv is $O\left(T2^{D+F}(n+m)\right)$.
FastPD can also compute interaction values up to order $s$ in complexity:

\[
O\left(T \times n \times 2^D \times \sum_{i=0}^{s} \binom{F}{i}\right)
\]

This complexity was not discussed in~\cite{FastPD}; we derived it ourselves and verified it empirically. Further details are provided in the supplementary material.

The \textsc{glex} package\footnote{\textsc{glex} GitHub: \url{https://github.com/PlantedML/glex}} implements the FastPD algorithm in R. Although FastPD theoretically supports separate background ($B$) and consumer ($C$) datasets, this option is not available in \textsc{glex}. For our experiments, we extended the package to support this setting while preserving its efficient implementation strategy.
FastPD can also compute PDP in $O(T2^D F(n+k))$ and Joint-PDP in $O(T2^D F^2(n+k^2))$. This is done by computing first-order interactions for PDP and up to second-order interactions for Joint-PDP, and then reconstructing the PDVs from these interactions.

\subsection{Contributions}
Building on our previous work, \algname~\cite{woodelf_paper}, we formulate PDP, Joint-PDP, and \aopdiv as cooperative-game metrics and introduce \algnamenew for their efficient computation on decision tree ensembles. Our contributions include:

\begin{enumerate}
    \item \textbf{PDP and Joint-PDP speedup}: 
    We construct consumer datasets that enable efficient computation of PDP and Joint-PDP using local-attribution methods. We apply this approach to both \algnamenew and FastPD, enhancing both methods. On large datasets, \algnamenew is about $6\times$ faster than FastPD.
        
    \item \textbf{Exponential complexity improvement for \aopdiv}: In the worst case, where all splits use distinct features ($F = 2^D$), FastPD \aopdiv complexity has a $2^{2^D}$ dependence, while \algnamenew depends only on $4^D$. We demonstrate this gap empirically: \algnamenew completes \aopdiv on $\num{400000}$ samples in just \num{5} minutes, whereas FastPD is estimated to take \num{1000000} years for this task.

    \item \textbf{Full PDP}: We introduce a novel method for selecting the $k$ plotted values in a PDP, yielding  \emph{full PDP} - a plot that displays the PDV for any possible feature value. Our speedups make it feasible to compute these plots for all model features within practical runtimes.

    \item \textbf{PDP Approximation}: We demonstrate how \algnamenew can be used to compute approximate PDPs, establishing a simple connection between approximate PDPs, exact PDPs, and Path-Dependence SHAP.
        
\end{enumerate}

\section{Preliminaries}
\label{sec:woodelf_intro}

\subsection{Pseudo Boolean Functions} 
\label{sec:PB_and_SHAP}

We remind the reader of Pseudo Boolean (PB) functions~\cite{pseudo_boolean_functions} and their representation in Weighted Disjunctive Normal Form (WDNF)~\cite{DBLP:conf/dna/ZhangJ05}, followed by the linearity property in this context.

\begin{definition}[PB Function]
A \emph{Pseudo Boolean (PB) function} is a function of the form
$F(x_1,\dots,x_h): \{0,1\}^h \to \mathbb{R}$.
\end{definition}

\begin{definition}[Positive and Negative Literal; Cube; WDNF]
A \emph{literal} is a Boolean variable $x_i$ (\emph{positive literal}) or its negation $\neg x_i$ (\emph{negative literal}).  
A \emph{cube} is a conjunction of literals.  
A \emph{Weighted Disjunctive Normal Form (WDNF)} formula represents a PB function as:
\[
F(x_1,\dots,x_h) = \sum_{k=1}^m w_k \cdot c_k(x_1,\dots,x_h),
\]
where each $c_k$ is a cube and $w_k \in \mathbb{R}$ is its weight.

\end{definition}

\begin{definition}[$S_k$, $S_k^+$, $S_k^-$] For a cube $c_k$, let $S_k$ denote its set of \emph{variables}, $S_k^+$ its \emph{positive literals}, and $S_k^-$ its \emph{negative literals}.
\end{definition}

\begin{definition}[Linearity]
\label{def:linearity}
A function $g$ over WDNF formulas satisfies \emph{linearity} if its value on a WDNF formula $F = \sum_{k=1}^m w_k \cdot c_k$ can be expressed as the weighted sum of its values on the individual cubes:
\[
g(F) = \sum_{k=1}^m w_k \cdot g(c_k)
\]
\end{definition}

\subsection{SHAP on Decision Trees and \algname}
\label{sec:woodelf_into}

The inputs to Background SHAP are a decision tree ensemble $M$, a consumer dataset $C$ containing the samples to be explained, and a background dataset $B$. For each consumer $c \in C$, the goal is to compute Shapley values~\cite{original_shapley_paper} for all of its features.

A cooperative game is defined over $M$, $B$, and $c$: when feature $f_i$ participates ($f_i \in S$), it takes the consumer value $c[f_i]$; when it is missing ($f_i \notin S$), it is iteratively replaced with values from each baseline $b \in B$. The model prediction is evaluated for every $b \in B$, and the results are averaged.

\algname, introduced in our previous work~\cite{woodelf_paper}, models this game as a PB function $F$, assigning 1 to $f_i$ when it participates and 0 when it is missing:
\begin{equation}
\label{model_to_wdnf_equation}
\frac{1}{|B|} \sum_{b \in B} M \!\left( \begin{cases} 
c[f_i] & \text{if } f_i \in S, \\
b[f_i] & \text{if } f_i \notin S
\end{cases} \right)
= F \!\left( \begin{cases} 
1 & \text{if } f_i \in S, \\
0 & \text{if } f_i \notin S
\end{cases} \right)
\end{equation}

\algname efficiently constructs a WDNF formula representing the PB function $F$. It then computes Background SHAP by applying a simple linear-time formula to the constructed WDNF. Path-Dependent SHAP is computed similarly, by constructing a slightly modified WDNF.

Importantly, \algname is not limited to Shapley values. This approach enables efficient computation of any function $v$ that takes a WDNF formula as input and satisfies the linearity property (e.g., Banzhaf values). 

A naive Background SHAP approach evaluates each consumer–background pair individually, resulting in $O(nm)$ complexity, where $|B| = m$, $|C| = n$, and $L$, $T$, and $D$ are treated as constants. \algname and PLTreeSHAP~\cite{linear_background_shap} use preprocessing to reduce this cost to $O(n + m)$.

More precisely, the complexity of \algname depends on the amortized cost of evaluating the function $v$ across all cubes over the variables $\{x_1,\dots,x_D\}$ (where $D$ is the tree depth), see Appendix~D of~\cite{woodelf_extended_paper}. Let $O(v_c^{avg})$ denote the average computational cost of evaluating $v$ on a single cube, and let $O(v_s)$ denote the number of unique feature subsets returned by $v$ across all cubes. For example, in the case of Shapley values, each variable in a cube is considered individually, so $O(v_c^{avg}) = O(v_s) = O(D)$. The overall complexity of \algname is:
\[
\text{Path-Dependent: } O(nTLv_s + TL3^D v_c^{avg})
\]
\[
\text{Background: } O(mTL + nTLv_s + TL3^D v_c^{avg})
\]

\section{Efficient PDP with \algnamenew}
\label{sec:faster_pdp}

In this section, we extend \algname to efficiently compute all required PDVs for plotting the PDPs of all features.

\paragraph{The simple approach.} Observe the strong correspondence between the PDV formula (Formula~\ref{PDV_formula}) and the model-prediction game formula (Formula~\ref{model_to_wdnf_equation}): assigning $S = \{f_j\}$ and $c[f_j] = v$ in Formula~\ref{model_to_wdnf_equation} renders the expressions identical.

To better understand this correspondence, consider the PDV of \textit{age} at $v = 50$. In the standard PDV computation, we set \textit{age} to 50 for all samples in the background dataset $B$, run the model, and average the predictions (Fig.~\ref{fig:brute_example}).

Equivalently, we can view this through a game-theoretic lens: imagine a consumer $c$
with $c[\textit{age}] = 50$ and the coalition where the feature \textit{age} participates and is fixed to 50, while all other features
are missing and filled with their values from $B$.
The model’s expected prediction where only \textit{age} participates
is exactly the PDV of \textit{age} at $v = 50$.

Leveraging this correspondence, \algname naturally extends to compute PDVs, yielding our \algnamenew. Let $M$ be a decision tree ensemble, $B$ its training set (or any dataset), and $\hat{\mathbf{f}}$ the set of features in $B$. For each $f_i \in \hat{\mathbf{f}}$, select $k$ values $v_{1_{f_i}},\dots,v_{k_{f_i}}$. We proceed as follows:
\begin{enumerate}
    \item Construct a consumer dataset $C$ with $k$ rows. The $t$-th row of $C$ contains the $t$-th selected value for each $f_i \in \hat{\mathbf{f}}$.
    \begin{equation}
    \label{eq:C_formation}
    C = (\:(v_{t_{f_i}})_{\forall f_i\in \hat{\mathbf{f}} } \:)_{\forall 1 \le t \le k}
    \end{equation}
    
    \item Use \algname with the background data $B$, the model $M$ and the consumer data $C$, to produce a WDNF for each row in $C$.
    \item For each WDNF formula $F$ and feature $f_i$, apply:
    \begin{equation}
    \label{eq:PDV_assignment}
    \PDV_{f_i} = F(f_i=1,\:\forall f_{j \neq i}=0)
    \end{equation}
    For row $t$ and feature $f_i$, this returns the average prediction when $f_i$ participates (set to $v_{t_{f_i}}$) and others are missing, i.e., $PDV_{v_{t_{f_i}},f_i}$.
\end{enumerate}

Since all assignments (including $F(f_i=1,\:\forall f_{j \neq i}=0)$) are functions over WDNF formulas that satisfy linearity, \algnamenew is guaranteed to compute them correctly.




\paragraph{The efficient approach.} Notice that Formula~\ref{eq:PDV_assignment} above violates the \emph{null player property}; that is, a variable whose truth value never affects the resulting weight might still receive a nonzero value. For example, the cube $3(\neg x_2)$ does not include the variable $x_5$, yet it still affects $\PDV_{x_5}$ (since the assignment $\{x_5 = 1,\; \forall f_{j \neq 5} = 0\}$ satisfies it). As a result, a single cube may influence all the features in the dataset, which significantly hurts \algname's performance.

Luckily, \emph{Centered Partial Dependence Values (CPDVs)}, obtained by subtracting the mean prediction on $B$ from each PDV, do satisfy the null player property:
\begin{equation}
\label{eq:cpdv_definition}
    \CPDV_{v,f_i} = \PDV_{v,f_i} - \frac{\sum_{b \in B} M(b)}{|B|}
\end{equation}

In the WDNF constructed by \algname, the mean term $\frac{1}{|B|}\sum_{b\in B}M(b)$ corresponds to the assignment where all variables are set to $0$ (all features are “missing” and taken from $B$). Combining this insight with Formula~\ref{eq:PDV_assignment} yields a simple formula for computing CPDV:
\begin{equation}
\label{eq:CPDV_assignment}
\CPDV_{f_i} = F(f_i=1,\:\forall f_{j \neq i}=0) - F(\forall f_j=0)
\end{equation}

Fig.~\ref{fig:brute_vs_woodelf} illustrates how \algnamenew and Formula~\ref{eq:CPDV_assignment} compute CPDVs efficiently.

\begin{figure}[t]
    \centering
    \includegraphics[width=0.99\linewidth]{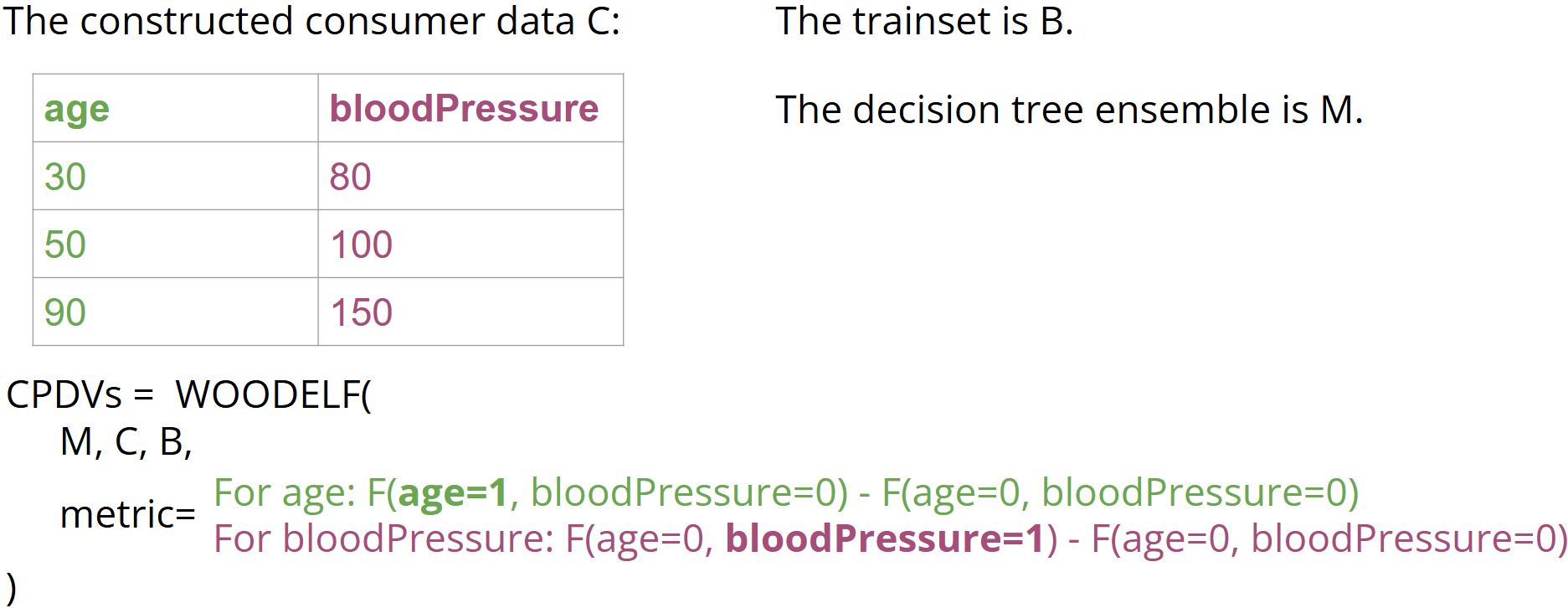}
    \caption{Using \algnamenew to compute the CPDVs required for the plots in Fig.~\ref{fig:pdp_example}.}
    \label{fig:brute_vs_woodelf}
\end{figure}

Finally, Formula~\ref{final_PDP_fomula} is equivalent to Formula~\ref{eq:CPDV_assignment} on WDNFs. It satisfies both linearity and the null player property, making its use within \algname both correct and efficient:
\begin{equation}
\CPDV_{f_i} = \sum\limits_{k=1}^m w_k \times 
\begin{cases}
1 & \text{if } (S_k^+ = \{f_i\}) \land (f_i \notin S_k^-) \\
-1 & \text{if } (f_i \in S_k^-) \land (S_k^+ = \emptyset) \\
0 & \text{otherwise}
\end{cases}
\label{final_PDP_fomula}
\end{equation}

The computation of CPDVs with the \textsc{woodelf} package is straightforward — Fig.~\ref{fig:CPDMetric} shows the entire code. The computation is performed by calling \algname with a \emph{CPDVMetric} instance, and requires no modifications to the package.


\begin{figure}[t]
\centering
\includegraphics[width=0.99\linewidth]{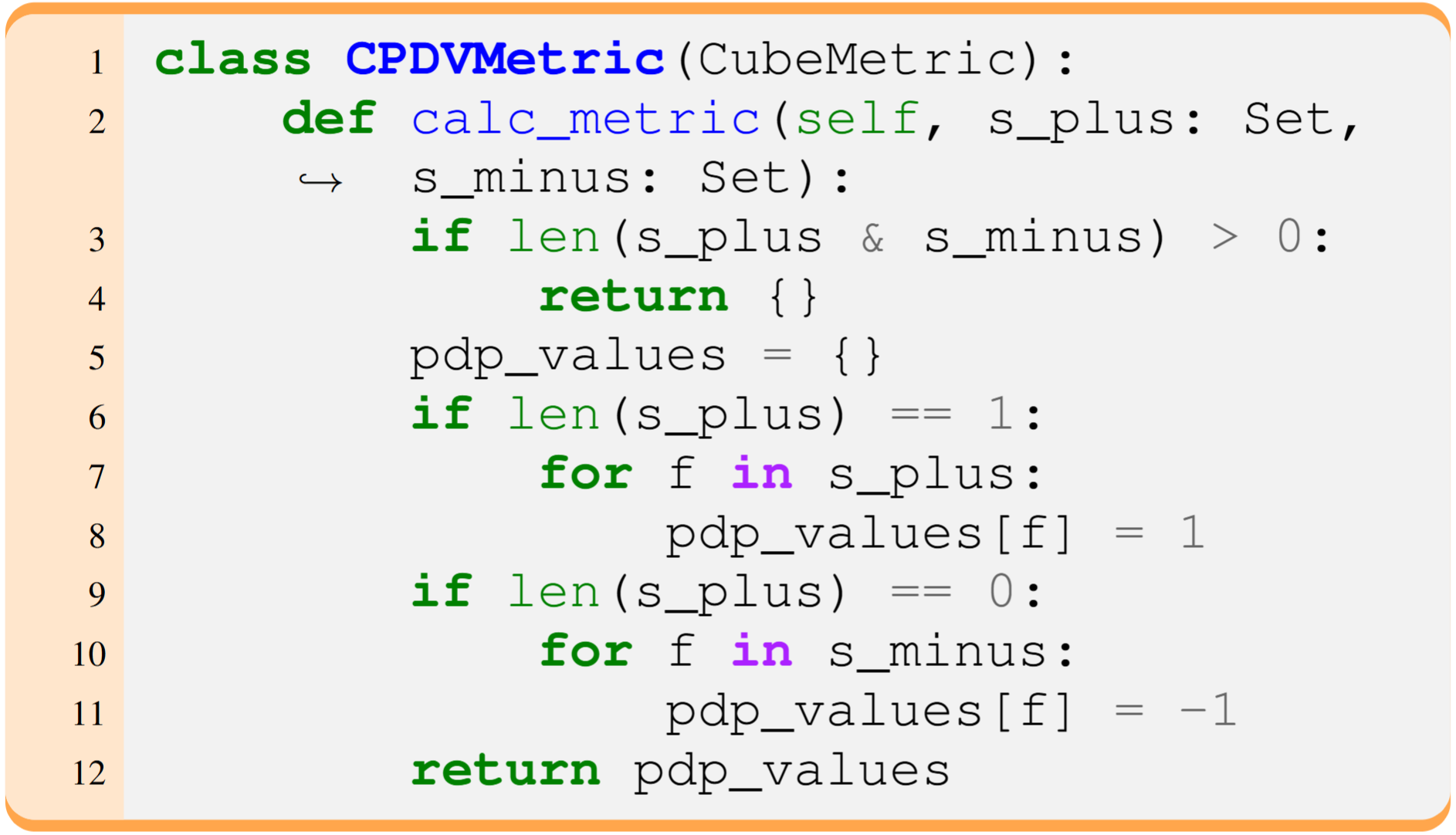}
\caption{
Python code for the class \emph{CPDVMetric} that implements Formula~\ref{final_PDP_fomula} for a single cube with weight 1. It inherits \algname's CubeMetric; a base class for metrics over WDNF formulas.}
\label{fig:CPDMetric}
\end{figure}

Recovering PDVs from CPDVs is immediate: add back the mean prediction on $B$. We now present the pseudo-code for our PDP algorithm. Alg.~\ref{alg:fast_pdp} receives a model $M$, a dataset $B$, and an integer $k$, and computes PDPs for all features in $B$, where each plot uses $k$ sampled points.

\begin{algorithm}[t]
\caption{Efficient Partial Dependence Plot}
\label{alg:fast_pdp}
\begin{algorithmic}[1]
\Function{WPDP}{$M, B, k$}
    \State Select $k$ possible values for each feature in $B$, and store them in $C$ (see Formula~\ref{eq:C_formation}). \label{line:c_construction}
    \State $CPDVs$ = \Call{\algname}{$M, C, B, \textit{CPDVMetric()}$} \label{line:woodelf_call}
    \State $PDVs = CPDVs + M(B).\textit{mean}()$ \label{line:decentralized}
    \State Use $PDVs$ to plot the PDP of each feature. \label{line:plot_pdp}
\EndFunction
\end{algorithmic}
\end{algorithm}

A complexity analysis of Alg.~\ref{alg:fast_pdp} is provided in the supplementary material. In short, the \algname call (line~\ref{line:woodelf_call}) dominates the runtime: we show that $O(v_c^{avg}) = O(v_s) = O(D)$ and that cubes with $|S_k^+| > 1$ can be ignored, leaving only $O(2^D D)$ relevant cubes. The complexity is:
\[
\text{Exact PDP}: O(nTL \;+\; kTLD \;+\; TL2^D D)
\]

\paragraph{Approximate PDP.} Node cover is the number of training samples that reach each node during training.
\cite{GreedyFunctionApproximation} introduces a PDP approximation based on node cover. \cite{FastPD} later observed that this uses the same approximation technique as Path-Dependent SHAP. In cooperative game terms, both methods are defined over the same game. \cite{FastPD} also discusses an inconsistency of this approach: it may produce different explanations for functionally equivalent trees.

\algnamenew leverages this connection to unify approximate and exact PDP under a single algorithm. Specifically, approximate PDP is obtained by computing the same PDV expression on the WDNF built with the Path-Dependent construction. In practice, this reduces to running the algorithm with an empty background dataset, so it defaults to the Path-Dependent setting (i.e., simply call: $\textit{WPDP}(M, \emptyset, k)$).





The complexity of Alg.~\ref{alg:fast_pdp} with an empty background dataset (see the Path-Dependent complexity in Sect.~\ref{sec:woodelf_into}):
\[
\text{Approximate PDP}: O(kTLD \;+\; TL2^D D)
\]

\section{Full PDP with \algnamenew}
\label{sec:full_pdp}

Standard PDPs depend on both the number ($k$) and choice of sampled values. A small $k$ can hide important patterns. For example, a TV‑rating model may reveal a sharp viewer spike at lunch (11:30–13:00), but a PDP sampled every four hours (6:00, 10:00, 14:00, 18:00) would miss it.

Some effects are hard to capture even with large $k$. For instance, a fraud‑detection model may learn that users who leave their annual salary at the default value of $\num{60000}$ are more likely to commit fraud. Even $k=\num{1000}$ might miss this, as it occurs at a single value within a wide range.

To address this, we introduce \emph{Full PDPs}, which capture all such behaviors in decision tree ensembles. For a given feature (e.g., \textit{age}), we take all split thresholds from tree nodes that split on that feature and use them as x-axis values. The rationale is simple: the ensemble changes its prediction only when \textit{age} crosses one of its thresholds. Between any two consecutive thresholds $th_1 < th_2$, the prediction remains constant:
\[
\forall_{th_1 < v < th_2} (\PDV_{v,\textit{age}} = \PDV_{th_1,\textit{age}})
\]
Thus, plotting only these points suffices.

Full PDPs reveal the complete piecewise‑constant structure of the decision tree ensemble, including narrow spikes and single‑value effects that sampled PDPs often miss. They are particularly useful for detecting overfitting and checking monotonicity. Fig.~\ref{fig:full_pdp_example} compares sampled PDPs with different $k$ values to a Full PDP, showing how the latter exactly recovers the true step function using only the feature’s split thresholds.

The main drawback is computational cost: large ensembles may have thousands of thresholds per feature. Previously, this made Full PDPs impractical. Both \algnamenew (via Alg.~\ref{alg:fast_pdp}) and FastPD handles large $k$ efficiently, enabling Full PDPs at practical speeds.

\begin{figure}[t]
    \centering
    \includegraphics[width=0.99\linewidth]{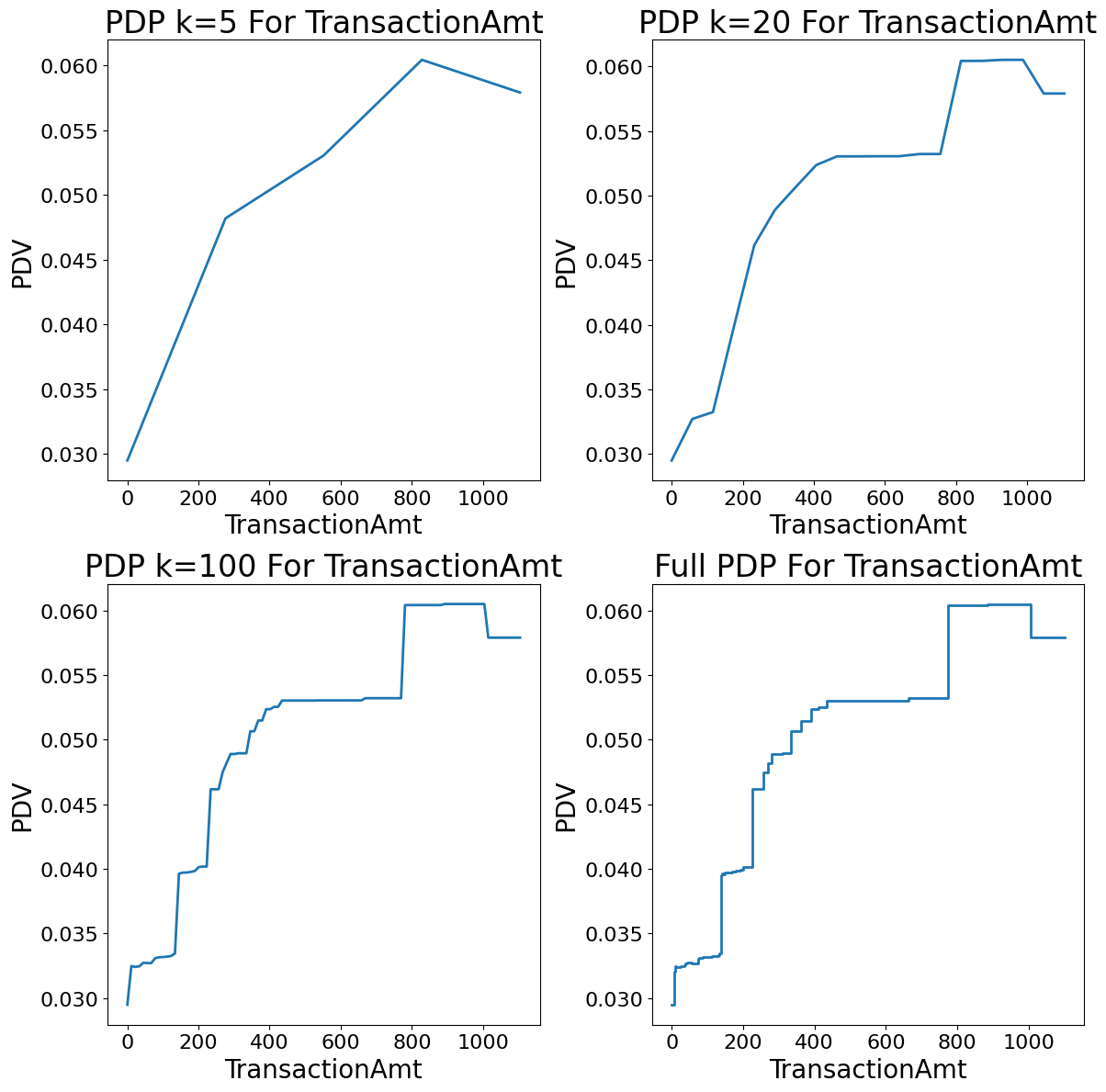}
    \caption{Comparison of PDPs for \textit{TransactionAmt} on the IEEE-CIS dataset using $k=5$, $k=20$, and $k=100$ uniformly spaced intervals versus the proposed \emph{Full PDP}. Standard PDPs approximate the 
    structure as $k$ increases but fail to recover exact discontinuities even at $k=100$. The Full PDP recovers the true step function using only 66 evaluation points corresponding to split thresholds.}
    \label{fig:full_pdp_example}
\end{figure}

\section{\aopdiv with \algnamenew}
\label{sec:pdivs}

In this section, we extend \algname to compute all Partial Dependence Interaction Values (PDIVs; see Def.~\ref{def:pdiv}). Our approach follows three steps:
\begin{enumerate}
    \item Derive a formula for PDIVs on a WDNF representation.
    \item Use this formula to implement a Python class inheriting \textit{CubeMetric} that computes all PDIVs of a cube.
    \item Run \algname with that \textit{CubeMetric}.
\end{enumerate}

For a pseudo-Boolean function $F$ and variable subset $X$, PDIV is defined by Formula~\ref{eq:pdiv_on_PB} (combining Def.~\ref{def:pdv_for_sets}, Def.~\ref{def:pdiv}, and Formula~\ref{eq:PDV_assignment}): 
\begin{equation}
    \label{eq:pdiv_on_PB}
    \PDIV_{X} = \sum_{S \subseteq X} (-1)^{|X|-|S|} F(\forall x_{i\in S}=1,\; \forall x_{i \notin S}=0)
\end{equation}

Because PDIV is linear, computing $PDIV_X$ on the WDNF reduces to determining the contribution of each cube $c_k$ separately and summing the results. We discard unsatisfiable cubes where $S^+_k \cap S^-_k \neq \emptyset$. A cube affects only subsets $X$ such that $S^+_k \subseteq X \subseteq S_k$:
\begin{itemize}
    \item If $X \not\subseteq S_k$, the interaction value is zero, since $X$ contains at least one feature that does not appear in the cube and is therefore ignored. Any interaction involving an ignored feature is zero.
    \item If $S^+_k \not\subseteq X$, at least one positive literal is falsified, so all assignments in Formula~\ref{eq:pdiv_on_PB} return zero.
\end{itemize}

For subsets $X$ where $S^+_k \subseteq X \subseteq S_k$, Formula~\ref{eq:pdiv_on_PB} sums the assignments of all $S \subseteq X$. The only such $S$ to satisfy the cube and affect the summation is $S=S^+_k$. Thus:
\begin{equation}
\PDIV_{X} = \sum\limits_{k=1}^m w_k \times 
\begin{cases}
(-1)^{|X|-|S_k^+|} & \text{if $S^+_k \subseteq X \subseteq S_k$}  \\
0 & \text{otherwise}
\end{cases}
\label{eq:pdiv_on_WDNF}
\end{equation}

    

\begin{figure}[t]
\centering
\includegraphics[width=0.99\linewidth]{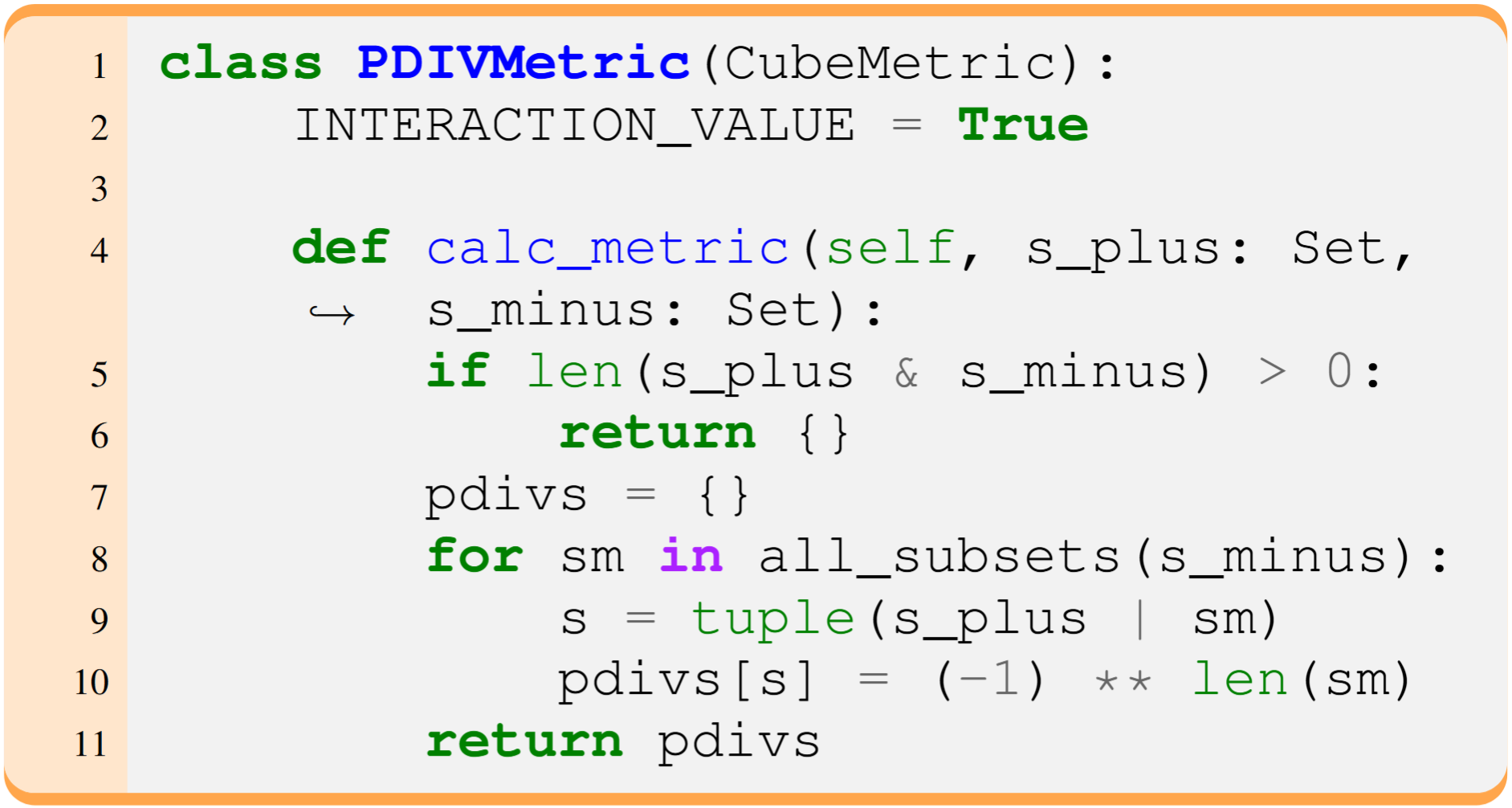}
\caption{
Python implementation of Formula~\ref{eq:pdiv_on_WDNF} for a single cube with weight 1. \texttt{all\_subsets} returns all subsets of a given set.}
\label{fig:PDIV}
\end{figure}

The \textit{PDIVMetric} code in Fig.~\ref{fig:PDIV} uses Formula~\ref{eq:pdiv_on_WDNF} and computes the effect of a single cube on all relevant subsets $X$. 
Running \algnamenew with \textit{PDIVMetric} is straightforward:
\[
\text{return } \textit{WOODELF}(M, C, C, \textit{PDIVMetric}())
\]
This simple \algnamenew call computes all the desired \aopdiv.
Here $C$ is used as both consumer and background data, but a separate background dataset can be provided (e.g., using the training set as background). Limiting consumer size is recommended for memory efficiency. An empty background dataset produces approximate PDIVs via the Path-Dependent approach. As \algname is GPU-friendly, our method naturally supports GPU acceleration.

\paragraph{Complexity.} For \textit{PDIVMetric}, $O(v_s)=O(2^D)$ and $O(v_c^{avg})=O((\frac{4}{3})^D)$, giving:
\[
\text{\aopdiv: } O(nTL2^D + TL4^D)
\]

This complexity is exponentially better than the previous state of the art, and for balanced decision trees with $2^D$ nodes it becomes polynomial in the input size. A detailed complexity analysis is provided in the supplementary material.

The intuition behind this complexity improvement is as follows. The naive algorithm treats the model as a black box, leading to complexity exponential in the total number of features. FastPD~\cite{FastPD} improves on this by processing each tree separately; its complexity is exponential in the number of features used per tree, resulting in a worst‑case dependence of $2^{2^D}$, which remains exponential in the input size. \algnamenew advances this line of improvement by treating each root‑to‑leaf path separately, yielding complexity exponential only in the number of features used per path, i.e., $2^D$.

\section{Joint-PDP with \algnamenew}
\label{sec:joint_PDP}

To compute all Joint PDVs, we use the algorithm from Sect.~\ref{sec:pdivs} and the identity:
\begin{equation}
\begin{split}
    \PDV_{\{a_i, b_j\}, \{f_a, f_b\}} = \PDIV_{\{a_i, b_j\}, \{f_a, f_b\}} + \\
    \PDIV_{\{a_i\}, \{f_a\}} + \PDIV_{\{b_j\}, \{f_b\}} + \PDIV_{\{\}, \{\}}
\end{split}
\label{eq:pair_PDV}
\end{equation}

The algorithm:
\begin{enumerate}
    \item Compute all PDIVs of order 2 and 1 using Sect.~\ref{sec:pdivs}.
    \item $\PDIV_{\{\},\{\}}$ is simply the average prediction on the background data. Compute it directly.
    \item Apply Formula~\ref{eq:pair_PDV} to obtain PDVs for all feature pairs.
\end{enumerate}

Deriving PDVs from PDIVs, as in Formula~\ref{eq:pair_PDV}, is also used in FastPD~\cite{FastPD}. Since only up to second-order interactions matter, step 1 can ignore cubes with $|S^+_k| > 2$, leaving only $O(2^D D^2)$ cubes to consider.

To construct the consumer dataset, we sample $k$ values per feature. For every pair of features $f_a$ and $f_b$, the dataset must contain all $k^2$ value pairs $(a_i, b_j)$. A naive construction allocates $k^2$ rows per feature, resulting in $O(k^2 f)$ rows overall. Specifically, we first construct a size-$k$ consumer data using the PDP procedure (Formula~\ref{eq:C_formation}). Then, for every feature $f_a$ and every sampled value $a_i$, we duplicate the constructed data $C$ and set feature $f_a$ to $a_i$ in all rows:
\[
\forall_{f_a\in F} \forall_{1 \le i \le k}\forall_{c \in C}:  c[f_a] \gets a_i
\]
Using a more efficient construction (see supplementary material), we reduce the dataset size to only $O(k^2 \log(f))$ rows and adapt \algnamenew to operate on this compressed representation.
Overall complexity (see notations in Fig.~\ref{fig:notations}):
\[
O(k^2f^2log(f) + nTL + k^2log(f) \! \cdot \!TLD^2 + TL2^DD^2)
\]

This consumer data construction is not specific to \algnamenew. We also use it to accelerate FastPD on Joint-PDP. 

As the PDIVs algorithm supports GPU acceleration, Joint-PDP is also GPU-friendly. Using full PDP sampling, one can plot \emph{full Joint-PDPs} showing the joint dependence for all values of $f_a$ and $f_b$.

\section{Experimental Results}
\label{sec:experimental_results}

The PDP implementation is available in \sckl via the \textit{partial\_dependence} function. This function implements both the exact method (\textit{method='brute'}) and the approximation method (\textit{method='recursion'}) of PDV computation.

FastPD is implemented in the \textsc{glex} R package. It is the current state-of-the-art tool for computing \aopdiv. The current implementation focuses on local feature attributions and supports only the case where the background and consumer datasets are identical. We extend the package to support distinct consumer and background datasets. Combined with our consumer-data constructions for PDP, Joint-PDP, and Full PDP, this enables efficient computation of these tasks within the FastPD framework.

Finally, we implemented \algnamenew in Python for all discussed tasks and empirically validated it by comparing its outputs to those of \sckl.


We compare \algnamenew running times against \sckl and FastPD on two large and widely used datasets: 
\begin{itemize}
    \item The IEEE-CIS fraud detection~\cite{kaggle_ieee_fraud_detection} dataset with \num{472432} rows and \num{397} features after one-hot encoding. 
    \item The KDD Cup 1999 intrusion detection~\cite{kdd_cup_1999_data} dataset with \num{4898431} rows and \num{127} features after one-hot encoding. 
\end{itemize}
For consistency, these are the same datasets we used in the original \algname paper~\cite{woodelf_paper}. Additional experiments on smaller datasets and varying tree depths are reported in the supplementary material.


We used \sckl’s \emph{HistGradientBoostingRegressor} model, which employs the histogram-based approach popularized by LightGBM~\cite{LightGBM}. In all experiments, models were trained with 100 trees, each limited to a maximum depth of 6, matching XGBoost’s widely used default \textit{max\_depth}. Experiments were conducted on Google Colab. For CPU training, we enabled the high-memory runtime with 50GB RAM (the default is 12GB). For GPU training, we used the T4 runtime type with high-memory enabled. All experiment notebooks are available in the \textsc{WoodelfExperiments} GitHub repository\footnote{The \textsc{WoodelfExperiments} repository containing all the notebooks used in our experiments: 
\url{https://github.com/ron-wettenstein/WoodelfExperiments/tree/main/woodelf_experiments/woodelf_pdp/IJCAI26}}.

\begin{table}[h!]
\scriptsize
\centering
\renewcommand{\arraystretch}{1.6}

\begin{tabularx}{\columnwidth}{X|l|l|l|l}
\multicolumn{5}{c}{\textbf{IEEE-CIS}} \\
\textbf{Task} & \textbf{scikit-learn} & \textbf{FastPD} & \multicolumn{2}{c}{\textbf{\algnamenew}} \\
 & & & \textbf{CPU} & \textbf{GPU} \\\hline
Exact PDP k=5 & 58.5 min & 99 sec & 14.5 sec & 13.1 sec \\
Exact PDP k=10 & 112.4 min & 99 sec & 15.2 sec & 10 sec \\
Exact PDP k=100 & 19 hours* & 99 sec & 14.8 sec & 9 sec \\
Exact full PDP & 59.3 min & 98 sec & 11.8 sec & 7 sec \\
Exact Joint-PDP k=5 & 35 days* & 130 sec & 19 sec & 25 sec \\
Approx. PDP k=5 & 3.3 sec & - & 4.6 sec & 5.4 sec \\
Approx. Joint-PDP k=5 & 19.4 min & - & 10.2 sec & 21.4 sec \\
\aopdiv n=10,000 & - & $6\! \cdot \! 10^4$ years* & 21 sec & 22.6 sec \\
\aopdiv & - & $2.9\! \cdot \! 10^6$ years* & 275 sec & 30.7 sec \\
\end{tabularx}

\vspace{0.8em}

\begin{tabularx}{\columnwidth}{X|l|l|l|l}
\multicolumn{5}{c}{\textbf{KDD Cup}} \\
\textbf{Task} & \textbf{scikit-learn} & \textbf{FastPD} & \multicolumn{2}{c}{\textbf{\algnamenew}} \\
 & & & \textbf{CPU} & \textbf{GPU} \\\hline
Exact PDP k=5 & 72.1 min & 520 sec & 95.6 sec & 33.4 sec \\
Exact PDP k=10 & 100.2 min & 519 sec & 95.9 sec & 33.5 sec \\
Exact PDP k=100 & 17 hours* & 519 sec & 96.2 sec & 33.6 sec \\
Exact full PDP & 116.5 min & 519 sec & 90.3 sec & 26.8 sec \\
Exact Joint-PDP k=5 & 9 days* & 568 sec & 95.2 sec & 37.5 sec \\
Approx. PDP k=5 & 5.5 sec & - & 6.5 sec & 8.7 sec \\
Approx. Joint-PDP k=5 & 10.1 min & - & 21.6 sec & 26 sec \\
\aopdiv n=10,000 & - & 15 days* & 80.6 sec & 24 sec \\
\aopdiv & - & 20 years* & 35 min & 45 sec** \\
\end{tabularx}

\caption{\label{performance_table}Performance comparison between scikit-learn, FastPD, and \algnamenew. $k$ is the number of sampled values per feature; $n$ is the number of rows. Values marked with * are estimates; details on estimation are provided in the supplementary material. Dashes indicate unavailable or non-applicable measurements. **Due to RAM limitations, GPU evaluation of KDD ``\aopdiv'' was conducted on an L4 runtime environment.}
\end{table}

As shown in Table~\ref{performance_table}, our approach outperforms the state of the art, with \algnamenew achieving a $6\times$ speedup over FastPD on PDP and Joint-PDP. Notably, both FastPD and \algnamenew significantly outperform \sckl, the most widely used tool for this task, with the gap becoming especially pronounced for large $k$: tasks that take hours or even several days using \sckl are completed by both approaches in under 10 minutes.

In \aopdiv, where \algnamenew offers an exponential complexity improvement over FastPD, the preformance gap is substantial. On IEEE-CIS, this improvement translates to a reduction in runtime from an estimated one million years with FastPD to just five minutes using our approach.

\section{Conclusion}
\label{sec:conclusion}

In this work, we extend our previous algorithm, \algname, originally designed for Shapley and Banzhaf values, into \algnamenew: a unified framework for efficiently computing PDPs, Joint-PDPs, and Any-Order PDIVs on decision tree ensembles. The two key ideas are to formulate these explanation tasks as linear metrics over WDNF formulas, and to construct compact consumer datasets that allow local-attribution algorithms to efficiently compute global dependence plots. \emph{Full PDPs} take the latter idea a step further by constructing a consumer dataset that exactly capture the complete piecewise-constant behavior of decision tree ensembles.

This perspective yields substantial computational gains, especially on large datasets. For PDP and Joint-PDP, \algnamenew is approximately $6\times$ faster than FastPD. For \aopdiv, where \algnamenew significantly improves the complexity, runtime decrease from several years to just a few minutes.

\section*{Acknowledgments}

This work was supported in part by the Israel Science Foundation grant 2410/22.

\bibliographystyle{named}
\bibliography{ijcai26}


\clearpage

\appendix

\section{Complexity Analysis}

This section presents the complexity analysis of the \emph{WPDP} algorithm (Alg.~\ref{alg:fast_pdp}) and the \algnamenew \aopdiv algorithm (see Sect.~\ref{sec:pdivs}). For the complexity analysis of the \algnamenew Joint-PDP algorithm, see Appendix~\ref{sec:joint_pdp_complexity}. The relevant notations appear in Fig.~\ref{fig:notations}, and the complexity of \algname is detailed in Sect.~\ref{sec:woodelf_into}.

\subsection{The PDP Algorithm}
In this section we analyze the complexity of the \emph{WPDP} function (Alg.~\ref{alg:fast_pdp}), which computes PDP.

Let $n$ be the number of rows in the dataset $B$. Our approach constructs a consumer dataset $C$ with $k$ rows, corresponding to the $k$ sampled values. Computing the average prediction on $B$ at line~\ref{line:decentralized} requires $O(nTD)$ time. Constructing $C$ at line~\ref{line:c_construction} of Alg.~\ref{alg:fast_pdp}, adding a constant to all CPDVs at line~\ref{line:decentralized}, and plotting them at line~\ref{line:plot_pdp} requires $O(kf)$ time, where $f$ is the number of features. Since features that are not used by the model can be ignored, we can assume the number of features is smaller than the ensemble size: $f < TL$. Therefore, the dominant cost arises from running \algname at line~\ref{line:woodelf_call}. 

The average complexity of \emph{CPDVMetric} and the number of feature subsets it returns are $O(v_c^{avg}) = O(v_s) = O(D)$. Therefore, in the standard \algname implementation, the call at line~\ref{line:woodelf_call} requires $O(nTL \;+\; kTLD \;+\; TL3^D D)$.

Luckily, Formula~\ref{final_PDP_fomula} offers an additional efficiency gain: all cubes $c_k$ with $|S^+_k| > 1$ make no contribution and can be ignored. Out of the $O(3^D)$ possible cubes of length up to $D$, only $O(2^D D)$ cubes contain at most one positive literal. This reduces the complexity of key steps in \algname (computing $M$ and $s$) from $O(TL3^D D)$ to $O(TL2^D D)$. The total complexity of this improved version is:
\[
\text{Exact PDP}: O(nTL \;+\; kTLD \;+\; TL2^D D)
\]

\subsection{\aopdiv}
\label{sec:any_order_pdiv_complexity}
In this section we analyze the complexity of the \algnamenew approach for computing Any-Order Partial Dependence Interaction values (see Sect.~\ref{sec:pdivs}). 
As discussed in Sect.~\ref{sec:woodelf_into}, when \algname computes values defined by a function $v$ under the Background approach, its runtime is
$O(mTL + nTL\,v_s + TL\,3^D\,v_c^{avg})$, where $v_c^{avg}$ is the average cost of evaluating $v$ on a cube and $v_s$ is the number of distinct feature subsets that $v$ outputs across all cubes.

To analyze \algnamenew with \emph{PDIVMetric} (Fig.~\ref{fig:PDIV}), we bound $v_c^{avg}$ and $v_s$ for \emph{PDIVMetric} over the $3^D$ cubes on variables $\{x_1,\dots,x_D\}$, where each variable $x_i \in \{x_1,\dots,x_D\}$ can appear positively, negatively, or not at all.
This setting corresponds to the hardest case in which features along a depth-$D$ path do not repeat.
If features repeat, the resulting cubes can only become shorter, and evaluating \emph{PDIVMetric} becomes easier: for example, adding $x_i$ to a cube that already contains $x_i$ leaves the cube unchanged, while adding $x_i$ to a cube that already contains $\neg x_i$ makes \emph{PDIVMetric} discard the cube in $O(1)$ time (see line~6 in Fig.~\ref{fig:PDIV}).

Among these $3^D$ cubes, the cube $(\neg x_1 \land \cdots \land \neg x_D)$ has the worst complexity.
On this cube, \emph{PDIVMetric} considers all subsets of the $D$ negated literals, resulting in $O(2^D)$ time; hence $v_c^{avg} \le 2^D$.
Across all cubes, the number of distinct feature subsets returned by \emph{PDIVMetric} is at most the number of subsets of $\{x_1,\dots,x_D\}$, implying $v_s = O(2^D)$.
Substituting these bounds yields an overall runtime of
$O\!\left(mTL + nTL\,2^D + TL\,6^D\right)$.

In practice, the bound is tighter: while \emph{PDIVMetric} has complexity $O(2^D)$ in the worst case, it is only $O(1)$ in the best case. Proposition~\ref{PDIV_vcavg} below proves that the average complexity of \emph{PDIVMetric} is $O((\frac{4}{3})^D)$.

\begin{proposition}
\label{PDIV_vcavg}
The average complexity of the function $v$ defined by \emph{PDIVMetric} (see Fig~\ref{fig:PDIV}) is:
\[
O(v_c^{avg}) = O\!\left(\left(\tfrac{4}{3}\right)^D\right).
\]
\end{proposition}
\begin{proof}
To compute the average complexity, we consider the total complexity across all $3^D$ cubes over the path variables $\{x_1, \dots, x_D\}$, where each variable $x_i$ may appear as a positive literal ($x_i \in S_k^+$), a negative literal ($x_i \in S_k^-$), or be absent ($x_i \notin S_k$). \emph{PDIVMetric} is evaluated exactly once for each cube in this $3^D$ family.

For any cube $c_k$ over $\{x_1, \dots, x_D\}$, \emph{PDIVMetric} visits all subsets $s$ such that $S^+_k \subseteq s \subseteq S_k$ and matches them to a value. Thus, each number computed by \emph{PDIVMetric} (one number in each loop iteration) can be uniquely described by the pair $(c_k, s)$. The total complexity is the number of such pairs generated over all cubes.  

We now show a one-to-one correspondence between all $(c_k, s)$ pairs and the set of patterns $p \in \{1,2,3,4\}^D$.
Given a pair $(c_k, s)$, construct a single pattern of length $D$ using the rules below:
\begin{enumerate}
    \item Set $p[i] = 1$ if $x_i \in S^+_k$
    \item Set $p[i] = 2$ if $(x_i \in S^-_k) \land (x_i \in s)$
    \item Set $p[i] = 3$ if $(x_i \in S^-_k) \land (x_i \notin s)$
    \item Set $p[i] = 4$ if $x_i \notin S_k$
\end{enumerate}

Note: Since we consider a path with no repeating features, each variable $x_i$ can appear only once in each cube. In particular, there is no $i$ such that $x_i \in S_k^{+}$ and $x_i \in S_k^{-}$.
Given a pattern $p$, construct a single pair $(c_k, s)$ using the rules below:
\begin{enumerate}
    \item Add $x_i$ to $S^+_k$ if $p[i] = 1$
    \item Add $x_i$ to $S^-_k$ if $p[i] = 2$ or $p[i] = 3$
    \item Add $x_i$ to $s$ if $p[i] = 1$ or $p[i] = 2$
\end{enumerate}

This proves the one-to-one correspondence, implying that $\{1,2,3,4\}^D$ and the set of $(c_k, s)$ pairs have the same size.
As the size of $\{1,2,3,4\}^D$ is $4^D$, the size of all the $(c_k, s)$ pairs is also $4^D$ and \emph{PDIVMetric} total complexity is $O(4^D)$. Thus, the average complexity on all $O(3^D)$ cubes is: 
\[
O(v_c^{avg}) = O\!\left(\left(\tfrac{4}{3}\right)^D\right).
\]
\end{proof}

A more refined analysis uses $O(v_c^{avg})=O((\frac{4}{3})^D)$ and yields an even sharper complexity:

\[
\text{\aopdiv (} C=B\text{)}: O(nTL2^D + TL4^D)
\]

The analysis above assumes that the consumer data $C$ is also used as the background data $B$. For a separate background dataset with $m$ rows, the complexity is:

\[
\text{\aopdiv (} C\neq B\text{)}:O(mTL + nTL2^D + TL4^D)
\]

\section{The Joint PDP Algorithm}
This section presents an efficient algorithm for computing joint partial dependence plots (Joint-PDPs). As introduced in Sect.~\ref{sec:joint_PDP}, we first construct the required consumer data (Alg~\ref{alg:joint_pdp_data}, visualized in Fig~\ref{fig:point_df_for_joint_pdp}), and then compute the corresponding Joint-PDP values using \algnamenew.

\subsection{Joint PDP Consumer Data Construction}
To construct the consumer data, we sample $k$ points for each feature and use Formula~\ref{eq:C_formation} (discussed in Sec.~\ref{sec:faster_pdp}). For any two features $f_a$ and $f_b$, let their sampled values be $\{a_1, \dots, a_k\}$ and $\{b_1, \dots, b_k\}$, respectively. The constructed dataset for Joint PDP must contain a row for every pair of sampled values $(a_i, b_j)$. A simple way to achieve this for the columns $f_a$ and $f_b$ is as follows:
\begin{itemize}
    \item Extend the $f_a$ column using $\textit{np.tile}(X[f_a], k)$. This repeats the entire column $k$ times, e.g. $[0,1,2] \rightarrow [0,1,2,0,1,2,0,1,2]$.
    \item Extend the $f_b$ column using $\textit{np.repeat}(X[f_b], k)$. This repeats each element $k$ times, e.g. $[0,1,2] \rightarrow [0,0,0,1,1,1,2,2,2]$.
\end{itemize}
The resulting dataset has size $O(k^2)$ and contains all possible $(a_i, b_j)$ pairs.

A straightforward way to extend this construction to multiple features is, for each feature $f_a$, to construct a data extending $f_a$ using $\textit{np.tile}(X[f_a], k)$ and extending every other feature $f_r$ using $\textit{np.repeat}(X[f_r], k)$. This process produces a dataset of size $k^2$ for each feature, leading to a total size of $k^2f$. While the constructed dataset achieves the required pairwise coverage, its size increases linearly with the number of features and scales poorly on datasets with large number of features.

Fortunately, a more efficient algorithm exists. Alg~\ref{alg:joint_pdp_data} constructs a dataset with the same coverage properties and size $k^2log_2(f)$. The algorithm encodes each feature index in binary and traverses its bits from most to least significant. For each bit, the column is extended using \emph{np.tile} if the bit is 0 and \emph{np.repeat} if the bit is 1. Fig.~\ref{fig:point_df_for_joint_pdp} illustrates the method on a small example.

Since any two feature indices differ in at least one bit, the constructed dataset always includes a continuous set of rows where one column uses \emph{np.tile} and the other uses \emph{np.repeat}, ensuring that all value pairs are represented.

\begin{algorithm}
\caption{Construct the Data For Joint-PDP}
\label{alg:joint_pdp_data}
\begin{algorithmic}[1]
\Function{ConstructJointPDPData}{$df$}
    \State $F = df.columns$
    \State $k = len(df)$
    \State $data = \{f_i: [] \: | \: \forall f_i \in F \}$
    \For{$f_i$ in $F$}:
        \For{$b$ in $bits(i, \lceil log_{2}( |F|) \rceil)$}: \label{line:bits_for}
            \If{$b = 0$}:
                \State $data[f_i]$.extends($\textit{np.tile}(df[f_i], k)$) \label{line:np_tile}
            \Else:
                \State $data[f_i]$.extends($\textit{np.repeat}(df[f_i], k)$) \label{line:np_repeat}
            \EndIf
        \EndFor
    \EndFor
    \State \Return pd.DataFrame($data$)
\EndFunction
\end{algorithmic}
\end{algorithm}

The \emph{bits}$(i, \lceil log_{2}( |F|) \rceil))$ function call in line~\ref{line:bits_for} returns the binary representation of $i$ starting from the most significant bit and ending with the least significant bit. The second argument specifies the minimum length of the representation, the function pads with zeros if needed to reach this length. For example, $bits(11, 6)$ returns $[0,0,1,0,1,1]$.

The Joint-PDVs can be computed using first and second order PDIVs (see Formula~\ref{eq:pair_PDV} and Sect.~\ref{sec:joint_PDP}). To compute the Joint-PDVs, we will run the \algnamenew PDIVs algorithm (see Sect.~\ref{sec:pdivs}) on the created data and get the first and second order PDIVs of all its rows.  

As size of the created data is $k^2log_2(f)$, \algnamenew will create $k^2log_2(f)$ Joint-PDVs for each pair of features (see Sect.~\ref{sec:joint_PDP}). However, many of these values are redundant. To keep only the required values we need to clip them, leaving only a part where one feature was extended using \textit{np.tile} and the other using \textit{np.repeat}. We do this for each feature pair by keeping only the part that corresponds to the first bit where their binary representations differ. Algorithm~\ref{alg:joint_pdp_clip} describes this clipping procedure. The algorithm takes as input the Joint-PDVs, the feature set $F$, and the number of sampled points $k$. It returns the clipped Joint-PDVs.

\begin{algorithm}
\caption{Clip the Return PDVs}
\label{alg:joint_pdp_clip}
\begin{algorithmic}[1]
\Function{ClipPDVs}{$pdvs, F, k$}
    \State $D = \lceil log_{2}( |F|) \rceil)$
    \For{$(f_i, f_j)$ in $pdvs$}:
        \State $h=0$
        \For{$b_i, b_j$ in $zip(bits(i, D), bits(j,D))$}: 
            \If{$b_i \neq b_j$}:
                \State $pdvs = pdvs[(f_i, f_j)][hk^2:(h+1)k^2]$
                \State \textbf{break}
            \EndIf
            \State $h = h + 1$
        \EndFor
    \EndFor
    \State \Return $pdvs$
\EndFunction
\end{algorithmic}
\end{algorithm}

\subsection{Joint-PDP Pseudo Code}

Alg.~\ref{alg:fast_joint_pdp} presents the complete Joint-PDP algorithm. It takes as input a decision tree ensemble $M$, a background dataset $B$, and a number $k$. The algorithm proceeds as follows:
\begin{enumerate}
    \item Sample \( k \) values for each feature (line~\ref{line:sample_k_points}).
    \item Build the Joint-PDP dataset using \emph{ConstructJointPDPData} (line~\ref{line:ConstructJointPDPData_call}).
    \item Compute all PDIVs of first and second order using \algnamenew (lines~\ref{line:PDIV_1_2_order_start}-\ref{line:PDIV_1_2_order_end}). This is done by invoking \algnamenew with the \textit{PDIVOrder1Or2} metric, whose implementation appears in Fig.~\ref{fig:PDIV_order_1_2}. 
    \item Compute $\PDIV_{\{\},\{\}}$, which corresponds to the average prediction over the background data (line~\ref{line:PDIV_empty_set}).
    \item Apply Formula~\ref{eq:pair_PDV} to compute Joint-PDVs (lines~\ref{line:joint_pdv_compute_start}-\ref{line:joint_pdv_compute_end}).
    \item Clip the resulting Joint-PDVs using \emph{ClipPDVs} and return the result (line~\ref{line:clip_pdvs}).
\end{enumerate}

\begin{algorithm}
\caption{Joint-PDPs}
\label{alg:fast_joint_pdp}
\begin{algorithmic}[1]
\Function{WJointPDPs}{$M, B, k$}
    \State Select $k$ possible values for each feature in $B$, and store them in $df$ (see Formula~\ref{eq:C_formation}). \label{line:sample_k_points}
    \State $C$ = \Call{ConstructJointPDPData}{$df$} \label{line:ConstructJointPDPData_call}
    \State $metric = \textit{PDIVOrder1Or2()}$ \label{line:PDIV_1_2_order_start}
    \State $PDIV$ = \Call{\algname}{$M, C, B, metric$} \label{line:PDIV_1_2_order_end}
    \State $PDIV[\emptyset] = M(B).\textit{mean}()$ \label{line:PDIV_empty_set}
    \State $PDV = \{\}$ \label{line:joint_pdv_compute_start}
    \For{$f_i$ in $F$}
        \For{$f_j$ in $F$}
            \If{$f_i \neq f_j$}
            
                \State $PDV[(f_i,f_j)]$ = $PDIV[(f_i, f_j)] + PDIV[(f_i)] + PDIV[(f_j)] + PDIV[\emptyset]$
            \EndIf
        \EndFor
    \EndFor \label{line:joint_pdv_compute_end}
    \State \Return \Call{ClipPDVs}{$PDV, B.columns, k$} \label{line:clip_pdvs}
\EndFunction
\end{algorithmic}
\end{algorithm}

    

\begin{figure}[h]
\centering
\includegraphics[width=0.99\linewidth]{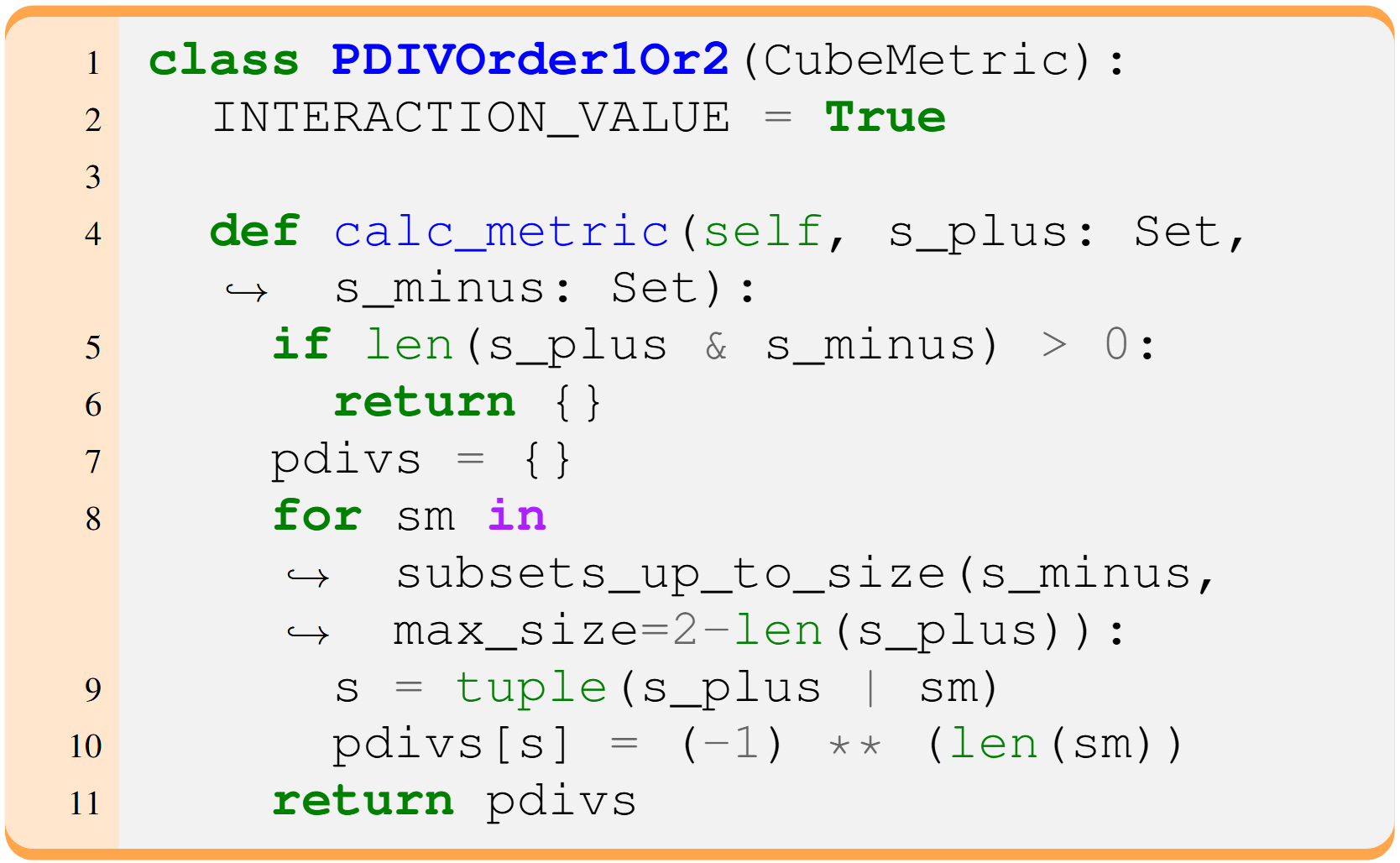}
\caption{
Python code for \emph{PDIVOrder1Or2}: a class that inherits \emph{CubeMetric} and computes all the PDIVs of first and second order for a single cube with weight 1. Its code closely follows that of \emph{PDIVMetric} (Fig.~\ref{fig:PDIV}). \texttt{subsets\_up\_to\_size} returns all subsets of a given set whose size is at most \textit{max\_size} (when \textit{max\_size} is negative, the function returns an empty list).}
\label{fig:PDIV_order_1_2}
\end{figure}

\begin{figure}[t]
    \centering
    \includegraphics[width=0.8\linewidth]{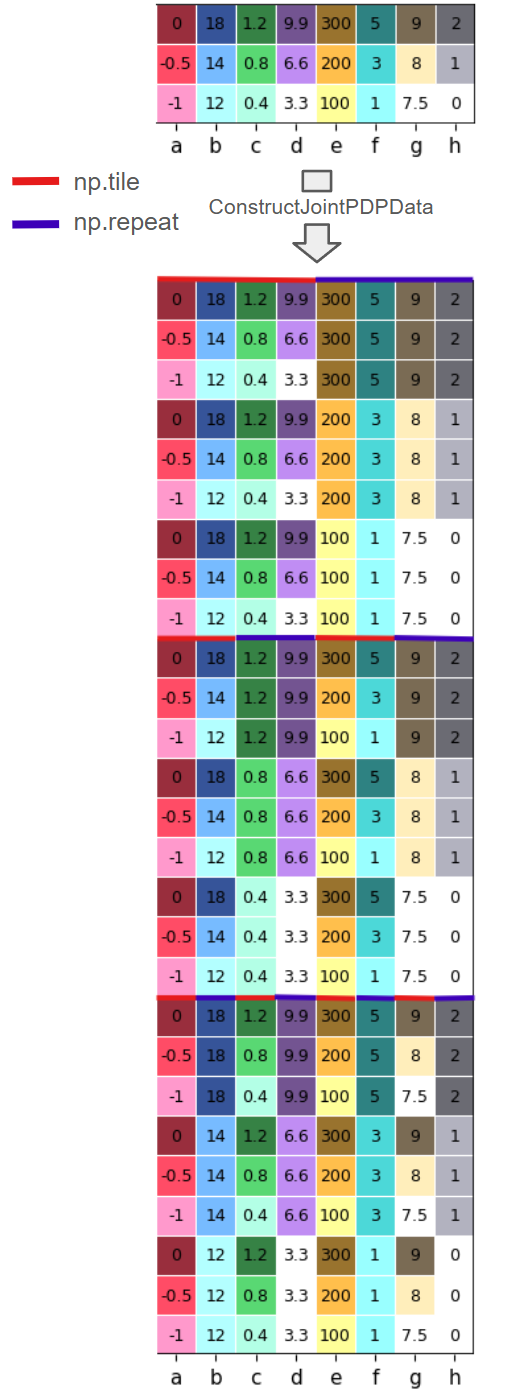}
    \caption{Consumer data for standard PDP (top) and the expanded dataset used for Joint-PDP (bottom). Columns are extended using \emph{np.tile} (red) or \emph{np.repeat} (blue), ensuring that all pairwise value combinations appear in the data. The standard PDP data has 3 rows and 8 columns, the Joint-PDP data has $3^2log_2(8)=27$ rows and 8 columns.}
    \label{fig:point_df_for_joint_pdp}
\end{figure}

\subsection{Joint-PDP Complexity Analysis}
\label{sec:joint_pdp_complexity}
We analyze the complexity of Alg.~\ref{alg:fast_joint_pdp} line by line:
\begin{enumerate}
    \item Line~\ref{line:sample_k_points} sample $k$ points for each feature in $B$, requiring $O(kf)$ time.
    \item Line~\ref{line:ConstructJointPDPData_call} has complexity $O(k^2 log(f))$, as \emph{ConstructJointPDPData} produces an output of size $k^2 log(f)$ and its complexity is linear with respect to its output size (Lines~\ref{line:np_tile} and~\ref{line:np_repeat} are the bottlenecks of Alg~\ref{alg:joint_pdp_data}).
    \item Line~\ref{line:PDIV_1_2_order_start} takes $O(1)$.
    \item In practice, line~\ref{line:PDIV_1_2_order_end} is the main bottleneck of the algorithm.
    To analyze the complexity of \algnamenew, we bound $O(v_s)$ and $O(v_c^{avg})$ when applying \emph{PDIVOrder1Or2} to all cubes over the variables $\{x_1,\dots,x_D\}$ (see Sect.~\ref{sec:woodelf_into} and Appendix~\ref{sec:any_order_pdiv_complexity} for details on these quantities).
    First, observe that any cube $c_k$ with more than two positive literals ($|S_k^{+}|>2$) contributes nothing and can be ignored.
    This leaves $O(2^D D^2)$ relevant cubes.
    
    We now show that the total cost of running \emph{PDIVOrder1Or2} on all these cubes is $O(2^D D^2)$.
    We partition the cubes into three groups and bound the total cost of each group:
    \begin{itemize}
        \item \textbf{No positive literals.}
        There are $O(2^D)$ such cubes.
        For each cube, \emph{PDIVOrder1Or2} iterates over all pairs of negated literals which costs $O(D^2)$ and over each negated literal individually which costs $O(D)$.
        The total cost is therefore $O(2^D D^2)$.
        
        \item \textbf{One positive literal.}
        There are $O(2^D D)$ such cubes.
        For each cube, \emph{PDIVOrder1Or2} iterates over all negated literals individually, with cost $O(D)$.
        The total cost is again $O(2^D D^2)$.
        
        \item \textbf{Two positive literals.}
        There are $O(2^D D^2)$ such cubes.
        In this case, \emph{PDIVOrder1Or2} runs in $O(1)$ time per cube, yielding a total cost of $O(2^D D^2)$.
    \end{itemize}
    
    Since \emph{PDIVOrder1Or2} computes PDIVs for individual features and feature pairs, it only produces values for singletons and pairs, giving $O(v_s)=O(D^2)$.
    Combining these bounds, letting $|B|=n$ and recalling that $|C| = k^2log(f)$, the overall complexity of line~\ref{line:PDIV_1_2_order_end} is $O(nTL + k^2log(f) \! \cdot \! TLD^2 + TL2^DD^2)$.

    \item Line~\ref{line:PDIV_empty_set} performs $n$ model predictions, taking $O(nTD)$ time.
    \item Lines~\ref{line:joint_pdv_compute_start}-\ref{line:joint_pdv_compute_end} iterate over all feature pairs and apply simple arithmetic on the \textit{PDIV} data of size $k^2log(f)$, resulting in $O(k^2 f^2 log(f))$ complexity.
    \item In line~\ref{line:clip_pdvs} the \emph{ClipPDVs} algorithm (Alg~\ref{alg:joint_pdp_clip}) clips the \textit{PDIV} data, creating a dataset with $k^2$ rows and $f^2$ columns in $O(k^2f^2)$ time.
\end{enumerate}

Combining the above line-by-line complexities yields a final complexity of:
\[
O(k^2f^2log(f) + nTL + k^2log(f)\! \cdot \!TLD^2 + TL2^DD^2)
\]

\section{Experimental Results}

This section provides technical details of our experiments, reports empirical correctness validations, and presents additional results on deeper trees and smaller datasets, complementing Sect.~\ref{sec:experimental_results}.

\begin{table*}[h!]
\begin{tabularx}{\textwidth}{r|X|r|r|r|r|r|r|r|r|r|r} 
\textbf{Dataset} & \textbf{Task} & \textbf{D=1} & \textbf{D=2} & \textbf{D=3} & \textbf{D=4} & \textbf{D=5} & \textbf{D=6} & \textbf{D=7} & \textbf{D=8} & \textbf{D=9} & \textbf{D=10} \\\hline 
 IEEE-CIS  &  Exact PDP k=5  &  3.3  &  3.6  &  4.3  &  5.4  &  7.2  &  11.1  &  17.7  &  31.0  &  60.2  &  142.1  \\ 
 IEEE-CIS  &  Exact PDP k=100  &  3.4  &  3.7  &  4.2  &  5.4  &  7.2  &  10.8  &  17.6  &  30.1  &  60.5  &  142.7  \\ 
 IEEE-CIS  &  Exact full PDP  &  0.9  &  1.3  &  2.0  &  3.1  &  5.0  &  8.5  &  15.7  &  27.9  &  58.9  &  141.2  \\ 
 IEEE-CIS  &  Exact Joint PDP k=5  &  4.0  &  4.4  &  5.2  &  6.3  &  8.6  &  12.9  &  23.3  &  47.8  &  135.1  &  445.7  \\ 
 IEEE-CIS  &  Approx. PDP k=5  &  3.2  &  3.7  &  4.3  &  5.4  &  7.3  &  10.7  &  17.9  &  29.8  &  60.7  &  144.0  \\ 
 IEEE-CIS  &  Approx. Joint PDP k=5  &  4.0  &  4.1  &  4.2  &  4.6  &  5.1  &  6.5  &  11.0  &  24.9  &  83.1  &  325.0  \\ 
  KDD Cup  &  Exact PDP k=5  &  10.3  &  13.7  &  19.9  &  29.5  &  55.0  &  81.8  &  120.8  &  182.2  &  312.6  &  502.5  \\ 
 KDD Cup  &  Exact PDP k=100  &  10.6  &  13.6  &  19.9  &  29.5  &  54.0  &  81.6  &  120.0  &  182.1  &  307.9  &  499.3  \\ 
 KDD Cup  &  Exact full PDP  &  6.4  &  9.9  &  15.7  &  25.5  &  49.8  &  77.7  &  116.4  &  179.2  &  303.7  &  495.8  \\ 
 KDD Cup  &  Exact Joint-PDP k=5  &  10.3  &  13.8  &  19.9  &  29.3  &  53.9  &  84.1  &  124.6  &  194.7  &  358.9  &  709.9  \\ 
 KDD Cup  &  Approx. PDP k=5  &  10.3  &  13.8  &  19.6  &  29.4  &  53.7  &  82.0  &  121.6  &  181.9  &  300.4  &  515.5  \\ 
 KDD Cup  &  Approx. Joint-PDP k=5  &  8.3  &  9.6  &  11.2  &  12.5  &  14.2  &  16.7  &  21.0  &  33.5  &  82.8  &  303.9  \\ 
\end{tabularx}
\caption{\label{higher_depth_table} We evaluated \algnamenew’s performance across different tree depths. A \emph{HistGradientBoostingRegressor} with 100 trees and no regularization (\textit{max\_leaf\_nodes}=\texttt{None}, \textit{min\_samples\_leaf}=1) is trained for each depth, and our algorithms are applied on the resulting ensembles. The experiments are conducted on a Google Colab high-RAM CPU environment. Each column in the table reports the running time for a specific depth (e.g., ``D=5'' corresponds to depth 5). All times are reported in seconds. As expected from the algorithm’s complexity, the running time approximately doubles with each increase in depth.}
\end{table*}

\subsection{Technical Experimental Details}
Below are additional technical details related to the experiments discussed in Sect.~\ref{sec:experimental_results}:
\begin{enumerate}
    \item \textbf{The effect of HistGradientBoostingRegressor on full PDP}: We used \sckl’s \emph{HistGradientBoostingRegressor}, which follows the histogram-based approach popularized by LightGBM~\cite{LightGBM}. All models were trained with 100 trees of depth 6. For efficiency, \emph{HistGradientBoostingRegressor} does not evaluate every possible split point, but instead considers up to 255 candidate thresholds per feature. This not only accelerates training but also reduces the complexity of computing full PDPs, since each feature can produce no more than 255 split values. As a result, the state of the art running times for full PDP and for PDP with k=5 (IEEE-CIS) or k=10 (KDD Cup) are comparable (see Table~\ref{performance_table}). An additional reason for this behavior is the large number of features in these datasets. In contrast, for datasets with fewer features or for models that do not rely on binned thresholds, computing the full PDP using the state-of-the-art approach can be substantially more expensive than computing a regular PDP. Our experiments on smaller datasets (see Table~\ref{small_datasets_preformance_table} in Sect.~\ref{sec:small_dataset_experiments}) highlight this difference.
    
    \item \textbf{PDIVs on all rows memory management}:  Storing the \aopdiv for all rows would have taken 125GB RAM on IEEE-CIS and 481GB RAM on KDD-Cup. To keep memory usage reasonable, we still compute the values for every row and all interaction orders, but only retain their averages in memory. The algorithm processes one tree at a time, averaging the results before moving on and keeping only the aggregated values.

    \item \textbf{Running Time Estimation}: Some of the reported running times for state-of-the-art approaches are estimated, see Sect.~\ref{sec:estimation} below for details.
\end{enumerate}

\subsection{Higher Depth}

Table~\ref{higher_depth_table} presents the running times of our algorithm at higher depths, showing that the runtime approximately doubles with each additional depth level.

We exclude the `\aopdiv' task from this table because the RAM requirements become prohibitively large at higher depths, even for a single tree. We report here the \aopdiv runs that were feasible within our available memory: for n=\num{10000} samples, runtime at depths 7, 8, and 9 reached 77 sec, 275 sec and 1242 sec on IEEE-CIS, and 105 sec, 202 sec and 532 sec on KDD-Cup, respectively.

\subsection{Smaller Datasets}
\label{sec:small_dataset_experiments}

For small-scale benchmarks, we use three standard datasets from \sckl: Diabetes (regression), Breast Cancer (binary classification), and California Housing (regression). These datasets are widely used in the literature and provide clean tabular settings for evaluating explanation methods. 

Details about these datasets are provided in Table~\ref{table_small_datasets}. Running time results are provided in Table~\ref{small_datasets_preformance_table}.

\begin{table}[h]
\centering
\scriptsize
\renewcommand{\arraystretch}{2.5}
\begin{tabular}{l|l|l|l}
\textbf{Name} & \textbf{Train Size} & \textbf{Test Size} & \textbf{Features} \\\hline
Diabetes & \num{353}  & \num{89} & \num{10} \\
Breast Cancer & \num{455} & \num{114} & \num{30} \\
California Housing & \num{16512} & \num{4128} & \num{8} \\
\end{tabular}
\caption{\label{table_small_datasets} Details of the small datasets used. We used 80\% of the data rows as a train set and 20\% as a test set. All our algorithms (PDP, Joint-PDP, \aopdiv) received the train set as an input.}
\end{table}

\begin{table}[h!]
\begin{tabularx}{\columnwidth}{l|X|r|r} 
\textbf{Dataset} & \textbf{Task} & \textbf{SOTA} & \textbf{\algnamenew}\\\hline 
 diabetes  &  Exact PDP k=5  &  0.25  &  0.39  \\ 
 diabetes  &  Exact PDP k=10  &  0.4  &  0.36  \\ 
 diabetes  &  Exact PDP k=100  &  2.73  &  0.4  \\ 
 diabetes  &  Exact full PDP  &  1.26  &  0.44  \\ 
 diabetes  &  Exact Joint-PDP k=5  &  3.53  &  0.72  \\ 
 diabetes  &  Approx. PDP k=5  &  0.02  &  0.4  \\ 
 diabetes  &  Approx. Joint-PDP k=5  &  0.1  &  0.52  \\ 
 diabetes  &  \aopdiv  &  0.3  &  0.92  \\ 
 cancer  &  Exact PDP k=5  &  0.67  &  0.4  \\ 
 cancer  &  Exact PDP k=10  &  1.34  &  0.39  \\ 
 cancer  &  Exact PDP k=100  &  12.19  &  0.42  \\ 
 cancer  &  Exact full PDP  &  3.06  &  0.45  \\ 
 cancer  &  Exact Joint-PDP k=5  &  46.03  &  0.7  \\ 
 cancer  &  Approx. PDP k=5  &  0.05  &  0.45  \\ 
 cancer  &  Approx. Joint-PDP k=5  &  1.05  &  0.49  \\ 
 cancer  &  \aopdiv  &  4.34  &  1.32  \\ 
 housing  &  Exact PDP k=5  &  2.29  &  1.7  \\ 
 housing  &  Exact PDP k=10  &  4.14  &  1.74  \\ 
 housing  &  Exact PDP k=100  &  38.63  &  1.64  \\ 
 housing  &  Exact full PDP  &  63.07  &  1.9  \\ 
 housing  &  Exact Joint-PDP k=5  &  36.84  &  3.03  \\ 
 housing  &  Approx. PDP k=5  &  0.03  &  1.72  \\ 
 housing  &  Approx. Joint-PDP k=5  &  0.12  &  1.81  \\ 
 housing  &  \aopdiv  &  34.98  &  6.35  \\ 
\end{tabularx}
\caption{\label{small_datasets_preformance_table} Performance comparison between the SOTA methods (\sckl for PDP and FastPD for PDIVs) and \algnamenew on additional, smaller and widely used datasets. For each dataset, a \emph{HistGradientBoostingRegressor} with 100 trees is trained and our algorithms are applied on the resulting ensembles. All experiments are executed on a Google Colab high-RAM CPU environment. All running times are reported in seconds.}
\end{table}

As shown in Table~\ref{small_datasets_preformance_table}, \algnamenew is faster than the SOTA in most tasks, including exact PDP (particularly when the number of plotted points $k$ is large), exact full PDP, and exact Joint-PDPs. However, when approximation is used, the SOTA can be faster on small datasets.

\subsection{SOTA Running Time Estimation}
\label{sec:estimation}
In this section, we describe how we estimated the SOTA running times reported in Table~\ref{performance_table}. We estimated the running times for the following tasks: `Exact PDP k=100', `Exact Joint-PDP k=5', `\aopdiv n=10,000' and `\aopdiv'. The methodology for each estimated value is explained below:
\paragraph{Exact PDP k=100} The SOTA runtime grows linearly with the number of sampled points $k$. This trend is clear from the actual (non-estimated) measurements of `Exact PDP k=5' and `Exact PDP k=10'. Therefore, to estimate `Exact PDP k=100' we take the measured runtime of `Exact PDP k=10' and multiply it by \num{10}.
\paragraph{Exact Joint-PDP k=5} The SOTA runtime for Joint-PDP grows linearly with the number of feature pairs. To estimate the running time of `Exact Joint-PDP k=5', we proceed as follows:
\begin{enumerate}
    \item We run the SOTA implementation on 10 feature pairs with $k=5$.
    \item We divide the total runtime by 10 to obtain the average runtime per pair.
    \item We multiply this average by the total number of feature pairs, which is $\frac{f(f-1)}{2}$.
\end{enumerate}
Equivalently, the total estimate is:
\[
\frac{running\_time\_on\_10\_pairs}{10} \cdot \frac{f(f-1)}{2} 
\]
In our runs, the measured time for 10 pairs was 382 seconds on the IEEE-CIS dataset and 1030 seconds on the KDD-Cup dataset.
\paragraph{\aopdiv} In both `\aopdiv n=10,000' and `\aopdiv', we used the same estimation methodology. 

First, we examined how FastPD scales with the number of participating features. We trained XGboost models with 3 to 9 features, 100 trees and depth 6 on a dataset containing \num{311029} rows and measured the runtime of FastPD. The observed running times were 48, 120, 223, 368, 742, 1217, and 2036 seconds. The runtime roughly doubled with each additional feature. This matches the theoretical FastPD complexity of $O(2^F nT)$ - where $F$ is the number of features used in a tree. 

Next, we used the 9-feature run and FastPD complexity to estimate $c$ - the cost of a single FastPD operation. With depth-6 trees, it is likely that all trees use all 9 features. This assumption is optimistic—if some trees use fewer features, the fitted value of $c$ underestimates the true operation cost, which in turn leads to an underestimated total running time. When using 9 features, FastPD required 2036 seconds on $n=311029$ rows and $T=100$ trees, leading to:
\[
2036 = 2^{9} \cdot 311029 \cdot 100 \cdot c
\]
which yields
\[
c = 8.67 \cdot 10^{-8} \text{ sec}
\]
The fitted constant $c$ corresponds to roughly 90 nanoseconds per operation, which is in line with expectations: on modern CPUs, each low-level numerical computation typically requires several nanoseconds to execute, and each “operation’’ in the complexity formula represents a small bundle of constant-time CPU steps.

To estimate FastPD runtime for the full tasks, we substituted the task parameters into the $O(2^F nT)$ complexity. The exponential term $2^FT$ was computed precisely by summing $2^{F_{T_i}}$ over all trees, where ${F_{T_i}}$ is the number of participating features in tree $T_i$:
\[
total\_exponential\_time = \sum_{T_i \in M} 2^{F_{T_i}}
\]

The resulting totals were \num{2199096386519040} for IEEE-CIS and  \num{1490305024} for KDD Cup. The much smaller value for KDD-Cup is likely due its fewer features and significant one-hot encoding. Finally, we multiplied $c \cdot n \cdot total\_exponential\_time$ to obtain the runtime estimates:
\begin{enumerate}
    \item `\aopdiv' IEEE-CIS (n=\num{472432}) is: 
    
    $c \cdot 2199096386519040 \cdot 472432 \approx 2854863$ years
    \item `\aopdiv n=10,000' IEEE-CIS is: 
    
    $c \cdot 2199096386519040 \cdot 10000 \approx 60429$ years
    \item `\aopdiv' KDD Cup (n=\num{4898431}) is: 
    
    $c \cdot 1490305024 \cdot 4898431 \approx 20$ years
    \item `\aopdiv n=10,000' KDD Cup is: 
    
    $c \cdot 1490305024 \cdot 10000 \approx 15$ days
    
\end{enumerate}

\subsection{Empirical Correctness Verification}

To validate the correctness of our approaches, we compared the PDV k=5 computed by \algnamenew (both with exact and approximate approaches) to those generated by \sckl. This comparison was carried out on the first \num{1000} rows of the IEEE-CIS train dataset. The results showed strong agreement, with all values differing by at most 0.00001. Note: The approximate PDV values of \sckl are centralized while the exact PDV values are not, \algnamenew supports both options.

Similar comparisons were done on 10 randomly chosen feature pairs to validate Joint-PDP correctness. The results are the same: all values differing by at most 0.00001.

Due to the exponential complexity of the state of the art, \aopdiv computation was validated on a smaller synthetic dataset with \num{10} features and \num{10000} rows. We implemented a direct PDIVs computation using Formula~\ref{PDIV_formula} and compared its output to our \algnamenew approach. On the first \num{100} rows when using the entire \num{10000} rows as a background data, the results matched closely, with all differences below 0.00001. 

We also demonstrated that approximate PDP is indeed an approximation by comparing its results to \sckl’s `brute' (exact) method and observing clear differences between the two.

\section{Computational Complexity of FastPD}
\label{app:complexity}

\cite{FastPD} show that the complexity of \texttt{FastPD} is
$O(T2^{D+F}(n_e + n_b))$, where $F$ is the number of
distinct split features in the tree, $n_e$ is the number of evaluation points,
and $n_b$ is the number of background samples.

When a maximum interaction degree $s \le F$ is imposed, only the
$\sum_{i=0}^{s}\binom{F}{i}$ subsets of size at most $s$
are processed in the evaluation step, reducing the dominant per-tree cost to
\begin{equation}
  \label{eq:complexity_s}
  O\!\left(T \cdot 2^D \cdot \sum_{i=0}^{s}\binom{F}{i} \cdot (n_e + n_b)\right).
\end{equation}
For fixed $D$, $T$, $n_b$, and $n_e$, this grows as $F^s/s!$ in $F$, so a
log--log plot of runtime against $F$ should have slope~$s$ (as $\log(F^s) = s \cdot \log F + \text{const}$).

\paragraph{Empirical validation.}
We trained XGBoost ensembles ($D=6$, $T=20$ trees) on synthetic data
designed so that all $F$ features are actively used in each tree.
We swept $F$ for three interaction orders and recorded wall-clock time
(three repetitions each):

\begin{table}[h]
  \centering
  \begin{tabular}{ccc}
    \toprule
    $s$ & $F$ range & Expected log--log slope \\
    \midrule
    $1$   & $3$--$35$ & $\approx 1$ \\
    $2$   & $3$--$28$ & $\approx 2$ \\
    $D=6$ & $3$--$12$ & $\approx 6$ \\
    \bottomrule
  \end{tabular}
  \caption{Benchmark configurations for each interaction order $s$.
           $F$ is the mean number of distinct split features per tree,
           measured from the fitted model.
           The expected slope follows from the $F^s/s!$ growth of
           $\sum_{i=0}^{s}\binom{F}{i}$.}
  \label{tab:complexity_benchmark}
\end{table}

The fitted log--log slopes matched the predictions in all three cases,
confirming Formula~\ref{eq:complexity_s}.
A complementary experiment fixed $F\approx10$ and swept $s=1,\ldots,D$
on a single model; the runtime tracked $\sum_{i=0}^{s}\binom{F}{i}$ closely across all orders.


\end{document}